%% file: main.tex
\documentclass[5p,twocolumn]{elsarticle}

\usepackage{amssymb}
\usepackage{hyperref}
\usepackage[switch]{lineno}
\usepackage{amsmath}
\usepackage[utf8]{inputenc}
\usepackage{xcolor}
\usepackage{graphicx}
\usepackage{subcaption}
\usepackage{amsthm}
\usepackage{multirow}
\usepackage{booktabs}
\usepackage{tikz}
\usetikzlibrary{patterns,arrows.meta,decorations.pathreplacing}
\usepackage[ruled,vlined,linesnumbered]{algorithm2e}
\usepackage{url}

\newtheorem{theorem}{Theorem}
\newtheorem{proposition}{Proposition}

\newcommand{\BG}{\texttt{BG}}
\newcommand{\FG}{\texttt{FG}}
\newcommand{\VIS}{\texttt{VISITED}}
\newcommand{\EAST}{\texttt{EAST}}
\newcommand{\minrow}{\operatorname{minrow}}
\DeclareMathOperator{\comp}{comp}

\journal{ArXiv}

\begin{document}

\begin{frontmatter}

% TODO: title to be decided with the authors (see candidates in the summary).
\title{TRUCO: Deterministic and Data-Race-Free Parallel Contour Tracing through Geometric Contour Ownership}

\author[address_uco]{Rafael~Mu\~noz-Salinas\corref{cor1}}
\ead{in1musar@uco.es}

\author[address_uco2]{Francisco J.~Romero-Ramirez}
\ead{fj.romero@uco.es}

\author[address_uco]{Manuel J.~Mar\'in-Jim\'enez}
\ead{mjmarin@uco.es}

\cortext[cor1]{Corresponding author}

\address[address_uco]{Department of Computing and Numerical Analysis, C\'ordoba University, Spain}
\address[address_uco2]{Department of Electronics and Computer Engineering, C\'ordoba University, Spain}

\begin{abstract}
Extracting closed pixel contours from a binary image, together with the hierarchy that records which contours enclose which, is a basic step in many computer vision pipelines, from fiducial marker detection to the vectorization of segmentation masks. The reference algorithm, by Suzuki and Abe, is inherently sequential: each contour is found by a raster scan that depends on the labels left by all previous contours, and its parent is obtained from the last contour met by the scan. This paper presents TRUCO (Threaded Raster Unrestricted Contour Ownership), a multi-core algorithm that returns exactly the output of Suzuki--Abe (the same contours, in the same order, with the same hierarchy) without locks, read-modify-write operations or a stitching phase. The method rests on two ideas. First, the image is divided into horizontal stripes, one per thread, and two rules, row-local writes and geometric contour ownership, let each thread trace complete contours independently of the others. Second, the tracing of the contours is separated from the computation of the hierarchy, which is recovered afterwards from one reference per contour; the trace therefore stores no contour numbers in the label image, which also makes it more efficient than that of the original algorithm when a single thread is used. We prove that the output is deterministic and identical to that of Suzuki--Abe, and that the algorithm is free of data races. Experiments on synthetic images, high-resolution photographs and segmentation masks, on a desktop and a server CPU, show that TRUCO scales with the number of threads and is several times faster than the fastest sequential methods, with and without the hierarchy. With a single thread, it is already faster than Suzuki--Abe.
\end{abstract}


\end{frontmatter}

%\linenumbers

%%%%%%%%%%%%%%%%%%%%%%%%%%%%%%%%%%%%%%%%%%%%%%%%%%%%%%%%%%%%%%%%%%%%%%%%%%%%%
\section{Introduction}
\label{sec:intro}
%%%%%%%%%%%%%%%%%%%%%%%%%%%%%%%%%%%%%%%%%%%%%%%%%%%%%%%%%%%%%%%%%%%%%%%%%%%%%

Binary contour extraction converts a binary image into the closed sequences of pixels that bound its connected components and their holes. It is a basic operation in fiducial marker detection \cite{ROMERORAMIREZ201838} and in industrial machine-vision workflows \cite{DiazCano2025MVS}, and it is also the usual way of turning the dense masks produced by modern segmentation models into polygons. Examples include contour-based instance segmentation \cite{Peng2020DeepSnake}, boundary-aware evaluation metrics \cite{Cheng2021BoundaryIoU}, the polygon annotations of large datasets \cite{Lin2014COCO} and the polygonal mapping of buildings from segmentation masks \cite{Xu2023HiSup,Zhang2025GCP}. In many of these tasks the contours alone are not enough: it is also necessary to know which contours are holes of which objects, and which objects lie inside which holes.  

The reference method is the algorithm of Suzuki and Abe \cite{SUZUKI198532}, which computes both the contours and their hierarchy and is the basis of the contour extraction functions of widely used image processing libraries \cite{bradski2000opencv}. With current image resolutions (4K and above) and the large number of masks produced per image, contour extraction can become a noticeable part of the processing time. However, the algorithm is sequential in two ways. First, contours are found by a raster scan that depends on the labels left by all the contours traced before, and each step of contour tracing depends on the previous one. Second, the parent of a new contour is obtained from the last contour met by the scan on the same row, identified by the number that the scan assigns to every contour.

A parallel version must therefore solve two problems. A contour is a global object that may cross any partition of the image, so every contour must be produced once and only once, without making the result depend on the relative speed of the threads. And the hierarchy must be recovered although the contour that determines the parent of a contour may have been traced by another thread, and although threads cannot share a global contour numbering without synchronizing. Existing parallel approaches split the image into tiles and stitch partial contours afterwards \cite{Garcia_2022,Garcia_2023,AlonsoJorda2026GPU}, which adds a merging phase whose cost grows with the number of contours crossing tile boundaries, or rely on hardware-specific architectures \cite{Chen1993,CHIA2003219,Gupta_2022,Gupta2024FPGA} or GPUs \cite{yin2025g}.

This paper presents TRUCO, a parallel algorithm for multi-core CPUs that solves both problems without stitching. The image is divided into horizontal stripes, one per thread, and each thread scans its stripe and traces the contours it finds, following them freely into the stripes below. Two simple rules make this safe. First, \emph{row-local writes}: a thread writes only on its own rows, and reads the rows of other threads only to distinguish background from foreground, which never changes; hence no thread depends on the progress of another, and the result does not depend on scheduling. Second, \emph{geometric ownership}: each contour is owned by the stripe that contains its top row, so ownership is decided locally, and a thread that sees a contour rise above its stripe simply abandons it. The hierarchy, which Suzuki--Abe obtains from the contour numbers stored in the label image, is recovered afterwards: each contour records a reference to a contour met during the scan of its start row, and a parallel pass over the traced contours turns these references into parents with the same rule as the sequential algorithm.

This additional pass, which Suzuki--Abe does not need, makes the method modular: applications that only need the contours, such as marker detection or the vectorization of masks without holes, simply skip it. When the hierarchy is required, the pass works on the contour points rather than on the image and is itself parallel, so the complete method remains faster than the sequential algorithm (Section~\ref{sec::experiments}).

The contributions of this paper are:
\begin{itemize}
  \item A stripe-parallel contour-tracing scheme, based on row-local writes and geometric ownership, that needs no stitching phase, no locks and no read-modify-write operations.
  \item A computation of the Suzuki--Abe hierarchy that is separated from the trace: the contours are traced with four fixed marks in a one-byte label image, and the hierarchy is reconstructed afterwards from per-contour references, without contour numbers. This separation is what allows the threads to share the label image without knowing the contours found by the other stripes, and it is advantageous by itself, independently of the parallel decomposition, since it also benefits the execution with a single thread.
  \item Proofs that the output is deterministic and identical to that of Suzuki--Abe, and that the algorithm is free of data races under the memory models of C, C++ and similar languages.
  \item An experimental evaluation on two CPUs and an open-source implementation\footnote{\url{https://github.com/rmsalinas/truco}}.% NOTE: change to OpenCV 4.15 if the PR is accepted
\end{itemize}

The remainder of the paper is organized as follows. Section~\ref{sec:related_work} reviews related work. Section~\ref{sec:suzuki_abe} summarizes the Suzuki--Abe algorithm, including the computation of the hierarchy. Section~\ref{sec:ourmethod} describes TRUCO and discusses its execution on shared memory. Section~\ref{sec::experiments} presents and discusses the experimental evaluation, and Section~\ref{sec:conclusions} concludes. \ref{app:proofs} proves that the output of TRUCO is deterministic and identical to that of Suzuki--Abe.

%%%%%%%%%%%%%%%%%%%%%%%%%%%%%%%%%%%%%%%%%%%%%%%%%%%%%%%%%%%%%%%%%%%%%%%%%%%%%
\section{Related Work}
\label{sec:related_work}
%%%%%%%%%%%%%%%%%%%%%%%%%%%%%%%%%%%%%%%%%%%%%%%%%%%%%%%%%%%%%%%%%%%%%%%%%%%%%

\subsection{Sequential Contour Tracing}
Early contour extraction methods follow the boundary of a region by local neighbourhood analysis \cite{Ren_2002}. Square tracing decides each turn from the state of the current pixel only and fails on some 8-connected configurations. Moore-neighbour tracing examines the eight neighbours of the current pixel in a fixed rotational order and, combined with a proper stopping criterion, follows any 8-connected contour. Seo et al.~\cite{seo_2016} reduce the number of neighbour inspections by classifying contour pixels into straight segments and corners. All these methods are sequential: the $(n+1)$-th contour pixel depends on the $n$-th, and the cost is linear in the length of the contour. Recent work on contour processing has focused on derived problems rather than on the tracing step itself, such as computing the contours of all the nodes of a component tree incrementally \cite{Silva2025ComponentTree} or storing large binary masks as chain codes and reconstructing them in parallel \cite{Strnad2024ChainCodes}.

\subsection{Contour Hierarchy and Connected-Component Labelling}
Suzuki and Abe \cite{SUZUKI198532} combine a raster scan with contour tracing to retrieve all contours of a binary image, outer contours and hole contours, together with their containment hierarchy, which they obtain from the contour numbers stored in an integer label image. It is the algorithm used by most image processing libraries \cite{bradski2000opencv}. Chang et al.~\cite{CHANG_2004} unify connected-component labelling and contour tracing in a single pass. Connected-component labelling (CCL) is closely related: it identifies the same 8-connected components whose outer contours are traced here, and a large body of work has optimized it with decision trees and block-based scans \cite{Grana_2010,Bolelli2020Spaghetti,Bolelli2022OneDAG}, run-based scans that also handle holes \cite{Lemaitre2022Holes,Boiangiu2025CCL}, and dedicated benchmarks \cite{Grana2016YACCLAB}. CCL produces a label image, however, not ordered contour pixel sequences or a hierarchy, and obtaining them from the labels still requires a contour-tracing pass and an analysis of the holes.

\subsection{Run-Length Encoding Approaches}
Run-based methods process maximal horizontal runs of foreground pixels instead of individual pixels \cite{KIM1988114,QUEK20001637}. Run-length representations are used for region labelling \cite{HE201725}, and early boundary-specific methods \cite{CAPSON1984109,CEDERBERG1979224} showed that contours can be extracted from raster-ordered runs in a single pass. The link-run approach builds contours by examining the overlap between runs of adjacent rows \cite{HE20091977}. These methods are fast on images with long runs, but the linking phase is sequential and their cost grows quickly on images with many short runs.

\subsection{Parallel Contour Extraction and Labelling}
Chen et al.~\cite{Chen1993} gave $O(\log N)$ PRAM algorithms to rank and linearize contour pixels, and Chia et al.~\cite{CHIA2003219} proposed a systolic architecture that generates chain codes in linear time. Gupta and Kar \cite{Gupta_2022} proposed contour tracing algorithms for torus-connected multiprocessors, and later mapped them onto a mesh of processing elements on an FPGA \cite{Gupta2024FPGA}. Garcia-Molla et al.\ parallelized Suzuki--Abe contour tracing on GPUs \cite{Garcia_2022} and on multi-core CPUs \cite{Garcia_2023} by partitioning the image into rectangular tiles, tracing contours inside each tile and joining the partial contours that cross tile boundaries in a later phase. Neither method computes the hierarchy, and the contours they return are not the same as those of Suzuki--Abe. Alonso-Jord\'a et al.~\cite{AlonsoJorda2026GPU} recently replaced the per-tile sequential tracing of \cite{Garcia_2022} by a fully data-parallel GPU scheme: every foreground pixel emits local ``triads'' (previous, current and next contour pixel), which are linked first inside each block and then across blocks. Its main goal is to avoid CPU--GPU transfers when contour extraction is embedded in a GPU inspection workflow \cite{DiazCano2025MVS}, and its run time grows with the number of contours and contour pixels. Yin et al.~\cite{yin2025g} proposed G-Contour, a GPU method for large VLSI layouts that first identifies objects with a parallel union--find structure and then extracts their contours. For connected components, parallel algorithms exist both for multi-core CPUs \cite{Cabaret2018} and for GPUs \cite{Playne2018,Allegretti2020GPU}, and a recent review with code compares the GPU approaches on a common benchmark \cite{Bolelli2026GPUReview}. In all these works, the main obstacle to parallel scaling is the resolution of equivalences between partitions, the counterpart of contour stitching.

TRUCO differs from tiling and triad-linking approaches \cite{Garcia_2022,Garcia_2023,AlonsoJorda2026GPU} in that no partial contour is ever produced: every contour is traced completely, by a single thread, in a single pass, so there is no merging phase, and the output, including the hierarchy, is exactly that of Suzuki--Abe. It differs from union--find and CCL-based approaches in that it requires no global data structure during tracing: the only shared object is the label image, and the hierarchy is recovered afterwards from the traced contours themselves. Finally, TRUCO targets multi-core CPUs, not GPUs: among the parallel contour tracers above, its direct counterpart is the multi-core method of \cite{Garcia_2023}, which is discussed in Section~\ref{sec:setup}.

\subsection{Contour Extraction in Learning-Based Pipelines}
The conversion of the dense masks of segmentation networks into polygons (Section~\ref{sec:intro}) relies on classical contour tracing, usually Suzuki--Abe, and needs the hierarchy whenever masks contain holes or nested regions. For instance, HiSup \cite{Xu2023HiSup} traces every predicted building mask with Suzuki--Abe before simplifying it, and the Global Collinearity-aware Polygonizer \cite{Zhang2025GCP} samples its initial polylines along the contours of the masks. Since tracing is applied to every instance of every image, its cost adds directly to the inference time. The output of TRUCO, including the hierarchy, is identical to that of Suzuki--Abe, so it can replace it without changing any downstream code. Section~\ref{sec:inria} evaluates TRUCO on the building masks of the Inria Aerial Image Labeling dataset, the kind of input that these methods polygonize.

%%%%%%%%%%%%%%%%%%%%%%%%%%%%%%%%%%%%%%%%%%%%%%%%%%%%%%%%%%%%%%%%%%%%%%%%%%%%%
\section{The Suzuki--Abe Algorithm}
\label{sec:suzuki_abe}
%%%%%%%%%%%%%%%%%%%%%%%%%%%%%%%%%%%%%%%%%%%%%%%%%%%%%%%%%%%%%%%%%%%%%%%%%%%%%

This section summarizes the Suzuki--Abe algorithm \cite{SUZUKI198532}, including the computation of the contour hierarchy, and fixes the definitions used in the rest of the paper.

\paragraph{Image, connectivity and contours}
Let $B$ be a binary image of $H\times W$ pixels, surrounded by a one-pixel frame of background, so that rows and columns are indexed $0,\dots,H+1$ and $0,\dots,W+1$. A pixel is \emph{foreground} if its value is non-zero and \emph{background} otherwise. As in \cite{SUZUKI198532}, foreground pixels are 8-connected and background pixels are 4-connected \cite{Rosenfeld1970}. A \emph{1-component} is a maximal 8-connected set of foreground pixels; a \emph{0-component} is a maximal 4-connected set of background pixels. The 0-component that contains the frame is the \emph{background}; any other 0-component is a \emph{hole}. For a pixel $q$, $\comp(q)$ denotes the component that contains it. Two components are \emph{adjacent} if they contain two 4-adjacent pixels. Following \cite{SUZUKI198532}, a component $S_2$ \emph{surrounds} a component $S_1$ if every 4-path from a pixel of $S_1$ to the frame contains a pixel of $S_2$. Every 1-component $A$ is adjacent to exactly one 0-component that \emph{surrounds} it, and every hole $Z$ is surrounded by exactly one 1-component; any other 0-component adjacent to $A$ is a hole surrounded by $A$ \cite{Rosenfeld1970,SUZUKI198532}. For a 1-component $A$ and an adjacent 0-component $Z$, the \emph{contour} $C(A,Z)$ between them is formed by the pixels of $A$ that are 4-adjacent to some pixel of $Z$; contour tracing visits them as a closed sequence, the \emph{pixel sequence} of the contour, in which a pixel may appear more than once. It is an \emph{outer contour} if $Z$ surrounds $A$ and a \emph{hole contour} if $Z$ is a hole surrounded by $A$. We write $C_A$ for the outer contour of $A$ and $C_Z$ for the hole contour of $Z$.

\paragraph{Raster order and start pixels}
Pixels are ordered in raster order: $(i,j)\prec(i',j')$ iff $i<i'$, or $i=i'$ and $j<j'$. For a set of pixels $X$, $\minrow(X)$ is the smallest row index of a pixel in $X$. For a 1-component $A$, its \emph{outer start} $p_A$ is its first pixel in raster order. For a hole $Z$, let $h_Z$ be its first pixel in raster order; its \emph{hole start} $p_Z$ is the left neighbour of $h_Z$, which is necessarily foreground. Two elementary geometric properties of the start pixels are used throughout the paper:
\begin{itemize}
  \item[(P0)] \emph{Surrounding component.} Let $p_A$ be the outer start of a 1-component $A$. Its west, north-west and north neighbours are background and belong to the 0-component that surrounds $A$.
  \item[(P1)] \emph{Top rows.} $\minrow(C_A)=\mathrm{row}(p_A)$ and $\minrow(C_Z)=\mathrm{row}(p_Z)-1$. Moreover, the north-east neighbour of $p_Z$ is foreground, and no pixel is both an outer start and a hole start.
\end{itemize}
For (P0), these neighbours are background because they precede $p_A$ and are 8-adjacent to it, and they are 4-connected, so they lie in one 0-component $Z'$; the vertical 4-path from the north neighbour of $p_A$ to the frame avoids $A$, so $Z'$ is not a hole of $A$ and therefore surrounds it.
For (P1), no pixel of $A$ lies above $p_A$, and $p_A\in C_A$ by (P0). For a hole $Z$ surrounded by $A$, the pixel $u$ above $h_Z$ precedes $h_Z$ and is 4-adjacent to it, so it is foreground, and it does not belong to a 1-component $A'$ surrounded by $Z$, since by (P0) the pixel above $p_{A'}$ would be a pixel of $Z$ above $h_Z$. Hence $u\in C_Z$, which gives $\minrow(C_Z)=\minrow(Z)-1$, and $u$ is the north-east neighbour of $p_Z$. In particular $p_Z$ is not the outer start of $A$, so no pixel is both an outer and a hole start.
Contours are ordered by their start pixels in raster order; we call this the \emph{scan order}. It is a total order, because no pixel is both an outer start and a hole start (P1).

\paragraph{Hierarchy}
The hierarchy is the tree defined by the surround relation \cite{SUZUKI198532}. The parent of an outer contour $C_A$ is the hole contour $C_Z$ of the hole $Z$ that surrounds $A$, or none if $A$ is surrounded by the background; the parent of a hole contour $C_Z$ is the outer contour $C_A$ of the 1-component $A$ that surrounds $Z$. We write $\pi(C)$ for the parent of $C$ and $\pi(C)=\bot$ for a root. In the rule of Table~\ref{tab:parentrule} below, the frame plays the role of a hole contour, identified with $\bot$.

\paragraph{The sequential algorithm}
Suzuki--Abe, summarized in Algorithm~\ref{alg:suzuki}, scans $B$ in raster order. It maintains a label image initialized to $B$ (foreground $=1$), a counter $NBD$ that numbers the contours found, and a variable $LNBD$, the number of the last contour met during the scan of the current row; at the beginning of each row $LNBD$ refers to the frame, which is treated as a hole contour. At each pixel $(i,j)$ it applies two tests:
\begin{enumerate}
  \item \textbf{Outer start:} $B_{i,j}=1$ (unvisited foreground) and $(i,j-1)$ is background. The contour between $\comp((i,j))$ and $\comp((i,j-1))$ is traced, starting from $(i,j)$ with $(i,j-1)$ as the first examined neighbour.
  \item \textbf{Hole start:} $B_{i,j}\ge 1$ (foreground, not marked as east) and $(i,j+1)$ is background. The contour between $\comp((i,j))$ and $\comp((i,j+1))$ is traced, starting with $(i,j+1)$.
\end{enumerate}
The first condition is checked first: a pixel that satisfies it is not tested for a hole start, which loses no contour because no pixel is both an outer and a hole start (P1).

Contour tracing (called border following in \cite{SUZUKI198532}) is Moore-neighbour tracing. At each visit of a pixel $p$, its neighbours are examined counter-clockwise until a foreground pixel is found, which is the next contour pixel. The search starts at $w$ at the first visit of the start pixel and, at any other visit, at the neighbour that follows the previous contour pixel in counter-clockwise order \cite{SUZUKI198532}. With directions numbered counter-clockwise from $\mathrm{E}=0$, a visit entered with direction $d_{in}$ and left with direction $d_{out}$ therefore examines the background neighbours in the cyclic range of directions $[d_{in}+5,\,d_{out})$. We use the equivalent formulation in which the trace stops when it returns to the start pixel and is about to repeat its first move, so the start pixel is visited twice. The search of the closing visit starts at the neighbour that follows the last contour pixel, exactly as the clockwise search of the start pixel in \cite{SUZUKI198532}, so the two formulations examine the same neighbours and produce the same labels. While tracing, each visited pixel is labelled: with $-NBD$ if its eastern neighbour is a background pixel examined during the neighbourhood search (an \emph{east} pixel), otherwise with $NBD$ if it was still unvisited. The negative label prevents the scan from starting a hole contour at the rightmost pixels of a contour that has already been traced. The parent of a new contour $C$ is obtained from the contour $D$ numbered $LNBD$ with the rule of Table~\ref{tab:parentrule}: $\pi(C)=D$ if $C$ and $D$ are of different types, and $\pi(C)=\pi(D)$ otherwise.

\begin{table}[t!]
\centering
\caption{Parent of a new contour $C$ given the contour $D$ met immediately before its start pixel \cite{SUZUKI198532}.}
\label{tab:parentrule}
\begin{tabular}{lcc}
\toprule
 & $D$ outer & $D$ hole \\
\midrule
$C$ outer & $\pi(D)$ & $D$ \\
$C$ hole  & $D$ & $\pi(D)$ \\
\bottomrule
\end{tabular}
\end{table}

\paragraph{Properties of the sequential algorithm}
TRUCO also relies on the following properties of the sequential algorithm, which are established in \cite{SUZUKI198532} or follow directly from the tracing procedure.
\begin{itemize}
  \item[(P2)] \emph{Contour tracing.} Let $q\in A$ and let $w\in Z$ be the left or the right neighbour of $q$, which is background. Contour tracing started at $q$, with $w$ as the first examined neighbour, terminates and visits exactly the pixels of $C(A,Z)$. It examines only background pixels of $Z$, and every pixel of $Z$ that is 4-adjacent to a pixel $q'$ of $C(A,Z)$ is examined during some visit of $q'$. The neighbours examined during a visit are determined by the directions in which the pixel is entered and left. The path of the tracer depends only on which pixels are background, not on the labels of foreground pixels.
  \item[(P3)] \emph{East labels.} A trace of $C(A,Z)$, complete or not, labels a pixel as east only if its eastern neighbour belongs to $Z$, and a complete trace labels as east every pixel of $A$ whose eastern neighbour belongs to $Z$.
  \item[(P4)] \emph{Output.} The algorithm returns exactly one contour per 1-component $A$, the outer contour $C_A$ traced from $p_A$, and one contour per hole $Z$, the hole contour $C_Z$ traced from $p_Z$, in scan order, together with the hierarchy $\pi$.
\end{itemize}

\begin{algorithm}[!tbp]
\renewcommand{\baselinestretch}{0.9}\selectfont
\SetAlgoLined
\DontPrintSemicolon
\KwIn{Binary image $B$ with background frame}
\KwOut{Contours $\mathcal{C}$ in scan order and parent map $\pi$}
 $\mathcal{C} \gets$ empty sequence;\ $NBD \gets 1$ \tcp*[r]{contour 1 is the frame ($\bot$)}
 \For{$i \gets 1$ \KwTo $H$}{
  $LNBD \gets 1$\;
  \For{$j \gets 1$ \KwTo $W$}{
    \If(\tcp*[f]{outer start}){$B_{i, j-1} = 0$ \textbf{and} $B_{i,j} = 1$}{
     $NBD \gets NBD + 1$;\ $c \gets \textbf{TraceContour}((i, j), (i, j-1), NBD)$\;
     $\pi(c)\gets$ rule of Table~\ref{tab:parentrule} with $D=LNBD$;\ append $c$ to $\mathcal{C}$\;
    }
    \ElseIf(\tcp*[f]{hole start}){$B_{i,j} \ge 1$ \textbf{and} $B_{i, j+1} = 0$}{
     \lIf{$B_{i,j}>1$}{$LNBD \gets B_{i,j}$}
     $NBD \gets NBD + 1$;\ $c \gets \textbf{TraceContour}((i, j), (i, j+1), NBD)$\;
     $\pi(c)\gets$ rule of Table~\ref{tab:parentrule} with $D=LNBD$;\ append $c$ to $\mathcal{C}$\;
    }
    \lIf{$B_{i,j}\notin\{0,1\}$}{$LNBD \gets |B_{i,j}|$}
  }
 }
\SetKwRepeat{Do}{do}{while}
\SetKwProg{Fn}{Function}{}{end}
 \Fn{\textbf{TraceContour}($p, w, NBD$)}{
    $c \gets$ empty sequence\;
    \Do{the stopping criterion is not met}{
        append $p$ to $c$\;
        $p' \gets$ next contour pixel: first foreground neighbour of $p$ in counter-clockwise order, starting at $w$ at the first visit and otherwise at the neighbour that follows the previous contour pixel\;
        \uIf{the eastern neighbour of $p$ is background and was examined}{ $B_{p} \gets -NBD$\; }
        \ElseIf{$B_{p} = 1$}{ $B_{p} \gets NBD$\; }
        $p \gets p'$\;
    }
    \Return $c$\;
 }
 \caption{Suzuki--Abe contour tracing with hierarchy (simplified from \cite{SUZUKI198532}).}
 \label{alg:suzuki}
\end{algorithm}

%%%%%%%%%%%%%%%%%%%%%%%%%%%%%%%%%%%%%%%%%%%%%%%%%%%%%%%%%%%%%%%%%%%%%%%%%%%%%
\section{The TRUCO Algorithm}
\label{sec:ourmethod}
%%%%%%%%%%%%%%%%%%%%%%%%%%%%%%%%%%%%%%%%%%%%%%%%%%%%%%%%%%%%%%%%%%%%%%%%%%%%%

TRUCO computes the same output as Suzuki--Abe (Section~\ref{sec:suzuki_abe}), the contours in scan order and their hierarchy, in three stages:
\begin{enumerate}
  \item \textbf{Parallel tracing} (Sections~\ref{sec:decomposition}--\ref{sec:rules}). The image is divided into horizontal stripes, and each thread finds and traces the contours that start in its stripe. Besides its pixel sequence, every contour $C$ receives a \emph{reference} $\rho(C)$: a pixel and a side that identify a contour $D$ from which the parent of $C$ follows with the rule of Suzuki--Abe.
  \item \textbf{Parallel reference resolution} (Section~\ref{sec:hierarchy}). The contour that contains each reference is found by a parallel pass over the points of all the traced contours.
  \item \textbf{Hierarchy assembly} (Section~\ref{sec:hierarchy}). The parent of every contour is obtained with the rule of Table~\ref{tab:parentrule}, and the contours are arranged in scan order.
\end{enumerate}
Stage~1 does not need contour identifiers in the label image, which uses one byte per pixel, and the threads never wait for each other. Stages~2 and~3 work only on the traced contours, not on the image.

Stages~2 and~3 are optional. When the hierarchy is not required, stage~1 alone returns the contours of Suzuki--Abe (Theorem~\ref{thm:main}), and stages~2 and~3 are skipped; if the scan order is also required, it is obtained by sorting only the contours that start on the rows shared by two stripes (Section~\ref{sec:hierarchy}). When the hierarchy is required, stage~2 adds a second pass over the contour points, which Suzuki--Abe does not need, but this pass is parallel and linear in the number of points, so it does not offset the gain of parallel tracing (Section~\ref{sec::experiments}).

For every number of stripes and every schedule, the output of the three stages is deterministic and identical to that of Suzuki--Abe. The formal statements and their proofs are given in \ref{app:proofs}.

\subsection{Domain Decomposition and State Space}
\label{sec:decomposition}

The $H$ inner rows of $B$ are divided into $K$ horizontal stripes of consecutive rows, $[s_k,e_k)$ for $k=0,\dots,K-1$, with $s_0=1$, $e_{K-1}=H+1$, $s_{k+1}=e_k$, and sizes that differ by at most one row, which requires $1\le K\le H$. $K$ is the number of worker threads, and the stripes are distributed among them by a parallel loop; any worker may process any stripe. In the following, ``thread $k$'' means the task that processes stripe $k$.

As in Suzuki--Abe, $B$ itself is used as the label image. It is the only object shared by all threads during stage~1, and its pixels take four values: \BG{}, \FG{}, \VIS{} and \EAST{}. The last two correspond to the positive and negative labels of Suzuki--Abe without the contour number, which is not needed because the parents are obtained from the references in stages~2 and~3. One byte per pixel suffices, instead of the integer label image of Suzuki--Abe. Two invariants hold during the whole execution:
\begin{itemize}
  \item \textbf{Background is never written.} Only foreground pixels receive marks, and both marks are different from \BG{}.
  \item \textbf{Marks are monotone.} A pixel can only change $\FG\to\VIS$, $\FG\to\EAST$ or $\VIS\to\EAST$.
\end{itemize}
Consequently, the background/foreground status of every pixel, which is the only information used to move the tracer (P2), is fixed during the whole execution, regardless of what threads write.

In addition, thread $k$ owns one private row $P_k$, which holds a copy of row $e_k$ (the first row of stripe $k+1$) taken when the thread starts, with every foreground value replaced by \FG{}.

\begin{figure}[t!]
\centering
\resizebox{\columnwidth}{!}{%
\begin{tikzpicture}[x=0.42cm,y=0.42cm,font=\scriptsize]
  \colorlet{tzero}{blue!70!black}
  \colorlet{tone}{orange!85!black}
  % stripes
  \fill[blue!10] (1,-1) rectangle (13,-5);
  \fill[orange!12] (1,-5) rectangle (13,-9);
  % foreground pixels
  \foreach \r in {3,4,7} \foreach \c in {7,...,11} \fill[gray!55] (\c,-\r) rectangle ++(1,-1);
  \foreach \r in {5,6} \foreach \c in {7,8,11} \fill[gray!55] (\c,-\r) rectangle ++(1,-1);
  \foreach \r in {6,7} \foreach \c in {2,3,4} \fill[gray!55] (\c,-\r) rectangle ++(1,-1);
  % row e_0, scanned by thread 0 in its private copy P_0
  \fill[pattern=north east lines, pattern color=tzero!35] (1,-5) rectangle (13,-6);
  \draw[gray!70!black, very thin, step={(1,1)}] (0,0) grid (14,-10);
  \draw[thick] (0,-1) -- (14,-1) (0,-5) -- (14,-5) (0,-9) -- (14,-9);
  % indices
  \foreach \c in {0,...,13} \node[above] at (\c+0.5,0) {\c};
  \foreach \r in {0,2,3,4,6,7,8} \node[left] at (0,-\r-0.5) {\r};
  \node[left] at (0,-1.5) {$s_0{=}1$};
  \node[left] at (0,-5.5) {$e_0{=}s_1{=}5$};
  \node[left] at (0,-9.5) {$e_1{=}9$};
  % accepted start pixels (filled: outer, open: hole) and aborted attempts (crosses)
  \fill[tzero] (7.5,-3.5) circle (0.25);
  \draw[tzero, very thick, fill=white] (8.5,-5.5) circle (0.25);
  \fill[tone] (2.5,-6.5) circle (0.25);
  \foreach \c in {7,11}
    \draw[tone, very thick] (\c+0.25,-5.25) -- (\c+0.75,-5.75) (\c+0.25,-5.75) -- (\c+0.75,-5.25);
  % stripes
  \draw[decorate, decoration={brace, amplitude=4pt}] (14.3,-1) -- (14.3,-5)
     node[midway, right=5pt, align=left] {stripe 0\\(thread 0)};
  \draw[decorate, decoration={brace, amplitude=4pt}] (14.3,-5.05) -- (14.3,-9)
     node[midway, right=5pt, align=left] {stripe 1\\(thread 1)};
\end{tikzpicture}}
\caption{Domain decomposition with $K=2$ on an $8\times 12$ image with its frame (foreground in grey). Thread $0$ writes marks on rows $1$--$4$ and on its private copy $P_0$ of row $e_0=5$ (hatched), where it makes hole attempts only; thread $1$ writes marks on rows $5$--$8$. Filled and open circles mark accepted outer and hole starts, coloured by owner; crosses mark aborted attempts.}
\label{fig:stripes}
\end{figure}
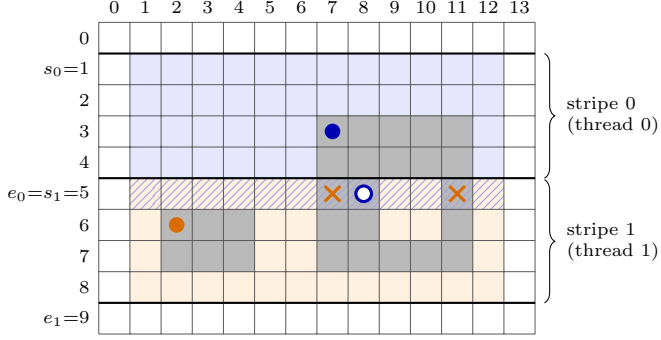

\subsection{Row-Local Writes and Geometric Ownership}
\label{sec:rules}

Thread $k$ scans rows $s_k$ to $e_k$ in raster order and starts traces with the tests of Suzuki--Abe, restricted to particular rows and evaluated on its own marks. Unlike Algorithm~\ref{alg:suzuki}, the hole test is applied even at a pixel that has just started an outer trace; since no pixel is both an outer and a hole start (P1), this does not change the set of start pixels. While tracing, it may move to any row $\ge s_k$, including rows of all the following stripes. Its behaviour is governed by two rules.

\paragraph{Rule 1: row-local writes}
Thread $k$ writes marks only on rows $[s_k,e_k)$ of $B$ and on its private row $P_k$, which stands for row $e_k$. When a trace leaves a pixel on row $e_k$, the mark is written in $P_k$; when it leaves a pixel on a row $>e_k$, no mark is written. The tests that start traces read marks only on rows $[s_k,e_k]$, that is, only marks written by thread $k$. Pixels of other rows are read only to decide whether they are background, which is invariant (Section~\ref{sec:decomposition}).

\paragraph{Rule 2: geometric ownership}
\begin{itemize}
  \item Outer attempts are made on rows $[s_k,e_k)$, at pixels $q$ with $M(q)=\FG$ whose left neighbour is background (as in Suzuki--Abe).
  \item Hole attempts are made on rows $(s_k,e_k]$, at pixels $q$ whose right neighbour is background, whose mark is not \EAST{}, and whose north-east neighbour is foreground. The last condition is a cheap necessary condition (P1) that discards most candidates before tracing.
  \item An outer trace started on row $r$ is \emph{accepted} only if it is completed and its top row is $r$; a hole trace started on row $r$ is accepted only if it is completed and its top row is $r-1$.
  \item A trace that moves to a row $<s_k$ is aborted immediately. This does not change the decision on that trace, which could not be accepted anyway, since accepted outer traces have top row $r\ge s_k$ and accepted hole traces have top row $r-1\ge s_k$.
\end{itemize}
Here $M(q)$ denotes the mark of $q$ as seen by thread $k$: its value in $B$ for rows $<e_k$ and in $P_k$ for row $e_k$. Aborted and rejected traces leave their marks in place; \ref{app:proofs} shows that these marks never prevent a correct start.

With these rules, a contour $C$ is accepted only by the thread whose stripe contains $\minrow(C)$. Outer contours start on their top row (P1). Hole contours start one row below their top row (P1), which is why thread $k$ also scans the first row of the next stripe, for holes only, and why it does not make hole attempts on its own first row $s_k$: those holes have their top row in stripe $k-1$. 

Figure~\ref{fig:stripes} illustrates both rules with $K=2$. The component on the right has its top row in stripe $0$, so thread $0$ owns both of its contours: it accepts the outer contour from $(3,7)$, although the contour extends to row $7$, and the hole contour, whose top row is $4$, from $p_Z=(5,8)$, which it scans in $P_0$. Thread $1$ also starts outer traces at $(5,7)$ and $(5,11)$, on the outer and on the hole contour of that component, but both move up to row $4<s_1$ and are aborted. The component on the left lies entirely in stripe $1$, and thread $1$ accepts its outer contour from $(6,2)$.

\paragraph{References}
When a contour $C$ is accepted, the thread records, besides its pixel sequence, its start pixel, its \emph{type} (outer or hole, given by the test that started the trace) and its \emph{reference} $\rho(C)$. The reference is a pair $\rho(C)=(q,\sigma)$, formed by a pixel $q$ of the start row and a side $\sigma\in\{\mathrm{W},\mathrm{E}\}$ of $q$ whose neighbour is background, and it identifies a contour met before the start pixel on that row. All three are known from the scan of the start row $r$, which is processed as a sequence of \emph{runs}, that is, maximal sequences of consecutive foreground pixels of the row, bounded by background on both sides:
\begin{itemize}
  \item For an outer contour started at the first pixel $(r,a)$ of a run, let $(r,b)$ be the last pixel of the previous run of row $r$. Then $\rho(C)=((r,b),\mathrm{E})$; if there is no previous run, $\rho(C)=\bot$.
  \item For a hole contour started at the last pixel $(r,b)$ of a run, let $(r,a)$ be the first pixel of that run. Then $\rho(C)=((r,a),\mathrm{W})$.
\end{itemize}
In both cases the neighbour of the reference pixel on the indicated side is background. Figure~\ref{fig:references} shows the references of a small image; the contours they identify and the parents obtained from them are discussed in Section~\ref{sec:hierarchy}. Algorithm~\ref{alg:parallel_worker} gives the complete procedure of stage~1.

\begin{figure}[t!]
\centering
\begin{tikzpicture}[x=0.6cm,y=0.6cm,font=\footnotesize]
  \foreach \r/\c in {1/1,1/2,1/3,1/4,1/5,1/6,1/7,2/1,2/7,3/1,3/3,3/5,3/7,4/1,4/7,5/1,5/2,5/3,5/4,5/5,5/6,5/7}
    \fill[gray!45] (\c,-\r) rectangle ++(1,-1);
  \draw[gray!70!black, very thin, step={(1,1)}] (0,0) grid (9,-7);
  \foreach \c in {0,...,8} \node[above] at (\c+0.5,0) {\c};
  \foreach \r in {0,...,6} \node[left] at (0,-\r-0.5) {\r};
  \node at (4.5,-1.5) {$A$};
  \node at (4.5,-2.5) {$Z$};
  \node at (3.5,-3.5) {$E$};
  \node at (5.5,-3.5) {$E'$};
  % references: side of the reference pixel, and the examined background neighbour w
  \draw[line width=2pt, red!80!black] (1,-2) -- (1,-3);
  \node[red!80!black] at (0.5,-2.5) {$w_Z$};
  \draw[line width=2pt, blue!80!black] (2,-3) -- (2,-4);
  \node[blue!80!black] at (2.5,-3.5) {$w_E$};
  \draw[line width=2pt, green!45!black] (4,-3) -- (4,-4);
  \node[green!45!black] at (4.5,-3.5) {$w_{E'}$};
\end{tikzpicture}
\caption{References in a $5\times 7$ image with its frame (foreground in grey). $A$ has a hole $Z$ that contains the single-pixel components $E$ and $E'$. Thick segments mark the references $\rho(C_Z)$, $\rho(C_E)$ and $\rho(C_{E'})$, and $w$ marks the background neighbour on the indicated side. The resulting parents are $\pi(C_A)=\bot$, $\pi(C_Z)=C_A$ and $\pi(C_E)=\pi(C_{E'})=C_Z$ (see the example in Section~\ref{sec:hierarchy}).}
\label{fig:references}
\end{figure}
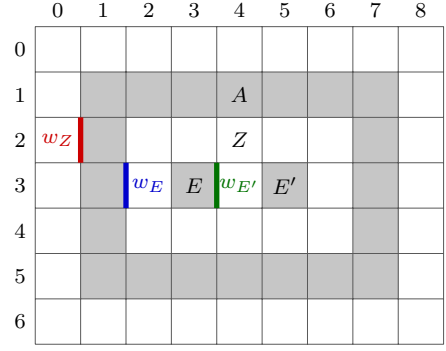

\begin{algorithm}[!tbp]
\renewcommand{\baselinestretch}{0.88}\selectfont
\SetAlgoLined
\DontPrintSemicolon
\SetKwProg{Fn}{Function}{}{end}
\KwIn{Shared label image $B$; stripe $[s,e)$; private row $P$}
\KwOut{List $\mathcal{C}_k$ of contours, each with its start pixel, type and reference}
 \For(\tcp*[f]{private copy of row $e$}){$x \gets 0$ \KwTo $W+1$}{ $P_x \gets (\mathrm{ld}(B_{e,x})\neq \BG)\ ?\ \FG : \BG$\; }
 \For{$r \gets s$ \KwTo $e$}{
   $L \gets (r<e)\ ?\ B_{r,\cdot} : P$ \tcp*[r]{marks of row $r$ seen by this thread}
   $b' \gets \bot$ \tcp*[r]{last pixel of the previous run}
   \ForEach{maximal run $[a,b]$ of foreground pixels of $L$}{
     \If(\tcp*[f]{outer attempt}){$r<e$ \textbf{and} $L_a = \FG$ \textbf{and} \textbf{Trace}$(r,a,\textit{outer})$}{
        append the trace to $\mathcal{C}_k$ with reference $(b'=\bot)\ ?\ \bot : ((r,b'),\mathrm{E})$\;
     }
     \If(\tcp*[f]{hole attempt}){$r>s$ \textbf{and} $L_b \neq \EAST$ \textbf{and} $B_{r-1,b+1}\neq \BG$ \textbf{and} \textbf{Trace}$(r,b,\textit{hole})$}{
        append the trace to $\mathcal{C}_k$ with reference $((r,a),\mathrm{W})$\;
     }
     $b' \gets b$\;
   }
 }
 \Fn{\textbf{Trace}$(r, c, t)$}{
    $p \gets (r,c)$;\ $m \gets r$;\ $\sigma \gets (t=\textit{outer})\ ?\ \mathrm{W} : \mathrm{E}$;\ $first \gets$ true\;
    \While{true}{
      append $p$ to the trace\;
      $d \gets$ direction of the first neighbour $n$ of $p$, counter-clockwise from $\sigma$, with $\mathrm{ld}(B_n)\neq \BG$\;
      \lIf(\tcp*[f]{isolated pixel}){none}{mark $p$ as \EAST; \Return $t=\textit{outer}$}
      $q \gets p + d$\;
      \lIf(\tcp*[f]{upward termination}){$\mathrm{row}(q) < s$}{\Return false}
      $m \gets \min(m, \mathrm{row}(q))$\;
      \If(\tcp*[f]{Rule 1}){$\mathrm{row}(p) \le e$}{
         \lIf{$\textsc{East}(\sigma,d)$}{$\mathrm{st}(M(p),\EAST)$}
         \lElseIf{$M(p) = \FG$}{$\mathrm{st}(M(p),\VIS)$}
      }
      \If(\tcp*[f]{Rule 2}){$p$ is the start pixel \textbf{and} \textbf{not} $first$ \textbf{and} $d = d_0$}{
         \Return $(t=\textit{outer})\ ?\ (m = r) : (m = r-1)$\;
      }
      \lIf{$first$}{$d_0 \gets d$;\ $first \gets$ false}
      $p \gets q$;\ $\sigma \gets d + 5 \pmod 8$ \tcp*[r]{neighbour after the previous pixel}
    }
 }
 \caption{Stage 1 of TRUCO for stripe $k$, with $s=s_k$, $e=e_k$ and $P=P_k$. $\mathrm{ld}$ and $\mathrm{st}$ are atomic loads and stores; $M(p)$ is the mark cell of $p$ for this thread ($B$ if $\mathrm{row}(p)<e$, $P$ if $\mathrm{row}(p)=e$). Directions are numbered counter-clockwise from $\mathrm{E}=0$; $\sigma$ is the first examined direction and $d$ the direction found, so the examined neighbours are those of the cyclic range $[\sigma,d)$. $\textsc{East}(\sigma,d)$ is the east-pixel condition of Suzuki--Abe: $\mathrm{E}\in[\sigma,d)$, i.e., $(-\sigma)\bmod 8<(d-\sigma)\bmod 8$.}
 \label{alg:parallel_worker}
\end{algorithm}

After all stripes have been processed, the lists $\mathcal{C}_0,\dots,\mathcal{C}_{K-1}$ are concatenated in stripe order. No other synchronization is used: there are no locks, no read-modify-write operations and no barriers apart from the end of the parallel loop.

\paragraph{Redundant work}
Accepted traces perform exactly the contour-tracing work of the sequential algorithm. In addition, a thread may start traces on contours owned by previous threads; these traces are aborted as soon as they move above $s_k$, or rejected when they complete. Their marks stay on the thread's rows, which prevents restarting the same contour from pixels already visited, but in the worst case every thread may partially trace every contour that crosses its first row, so the redundant work can grow with $K$. Section~\ref{sec::stresstest} evaluates images designed to maximize this effect.

\subsection{Hierarchy Reconstruction}
\label{sec:hierarchy}

In Suzuki--Abe the parent of a new contour is obtained from the contour numbered $LNBD$, the last contour met in the current row before the start pixel. TRUCO cannot read this number: the label image does not store contour numbers, and the contour met before the start pixel may be owned by another thread. The reference $\rho(C)$ replaces it. The referenced contour defined below is not always the contour numbered $LNBD$ (for a hole contour, a pixel inside the run of its start pixel may have been labelled by another contour), but it yields the same parent. For a reference $\rho(C)=(q,\sigma)$, let $w$ be the neighbour of $q$ on side $\sigma$, and define the \emph{referenced contour}
\[
  D(C)=C\bigl(\comp(q),\comp(w)\bigr).
\]
Since $q$ is foreground, $w$ is background and both are 4-adjacent, $D(C)$ is well defined, contains $q$, and is unique: there is exactly one contour between two adjacent components. The following properties justify stages~2 and~3.

\paragraph{The parent rule}
Let $C$ be an outer contour $C_A$ with $\rho(C)=((r,b),\mathrm{E})$. The pixels between $b$ and $p_A$ on row $r$ are background and 4-connected, so they belong to one 0-component $Z_0$, which contains the left neighbour of $p_A$ and therefore surrounds $A$ (P0). Thus $\pi(C)$ is the hole contour of $Z_0$, or $\bot$ if $Z_0$ is the background. $Z_0$ is also adjacent to $A'=\comp((r,b))$, so one of two cases holds. If $Z_0$ is a hole surrounded by $A'$, then $D(C)=C_{Z_0}$ is a hole contour and $\pi(C)=D(C)$. Otherwise $Z_0$ surrounds $A'$, $D(C)=C_{A'}$ is an outer contour, and $\pi(D(C))$ is again the hole contour of $Z_0$ or $\bot$, so $\pi(C)=\pi(D(C))$. Note that $Z_0$ may be a hole in the second case, when it surrounds both $A'$ and $A$. If $\rho(C)=\bot$, all pixels to the left of $p_A$ on row $r$ are background and connected to the frame, so $A$ is surrounded by the background and $\pi(C)=\bot$.
Let now $C$ be a hole contour $C_Z$ with $\rho(C)=((r,a),\mathrm{W})$, so that $a$ belongs to the 1-component $A$ that surrounds $Z$. The 0-component $Z_1$ of the left neighbour of $(r,a)$ is adjacent to $A$, so either it surrounds $A$, and then $D(C)=C_A$ and $\pi(C)=D(C)$, or it is a hole surrounded by $A$, and then $D(C)=C_{Z_1}$ and $\pi(C)=C_A=\pi(D(C))$.
In all cases,
\begin{equation}
  \pi(C)=\begin{cases}
    \bot & \text{if } \rho(C)=\bot,\\
    D(C) & \text{if } C \text{ and } D(C) \text{ are of different types},\\
    \pi(D(C)) & \text{otherwise},
  \end{cases}
  \label{eq:parent}
\end{equation}
which is the rule of Table~\ref{tab:parentrule} with $D(C)$ in place of the contour numbered $LNBD$. Moreover, $D(C)$ precedes $C$ in scan order. For an outer contour, the start pixel of $D(C)$ precedes or equals $(r,b)$, which precedes $p_A$. For a hole contour, either $D(C)=C_A$ and $p_A\preceq(r,a)\preceq p_Z$ with $p_A\neq p_Z$ (P1), or $D(C)=C_{Z_1}$, whose start pixel precedes the left neighbour of $(r,a)$, and hence $p_Z$.

\paragraph{Example}
Figure~\ref{fig:references} shows a component $A$ with a hole $Z$ that contains two single-pixel components, $E$ and $E'$. Row $1$ has a single run, so $\rho(C_A)=\bot$ and $C_A$ is a root. The hole contour $C_Z$ starts at $p_Z=(2,1)$, the end of the first run of row $2$, and its reference is the first pixel of that run, on its west side: $\rho(C_Z)=((2,1),\mathrm{W})$. The west neighbour belongs to the background, so $D(C_Z)=C_A$, which is of a different type, and $\pi(C_Z)=C_A$. On row $3$, $C_E$ starts at $(3,3)$ and its reference is the end of the previous run, on its east side: $\rho(C_E)=((3,1),\mathrm{E})$. The east neighbour $(3,2)$ belongs to $Z$, a hole surrounded by $A$, so $D(C_E)=C_Z$ and $\pi(C_E)=C_Z$. Finally, $C_{E'}$ starts at $(3,5)$ and $\rho(C_{E'})=((3,3),\mathrm{E})$. The east neighbour $(3,4)$ also belongs to $Z$, but $Z$ surrounds $E$, so $D(C_{E'})=C_E$ is an outer contour, like $C_{E'}$, and $\pi(C_{E'})=\pi(C_E)=C_Z$. The referenced contour is thus not always the parent, and it may belong to a component other than that of $C$; only the rule~\eqref{eq:parent} turns it into the parent. The side is also necessary: the pixel $(3,1)$ lies on both $C_A$ and $C_Z$, but only $C_Z$ examines its east neighbour, so stage~2 finds $D(C_E)=C_Z$ and no other contour.

\paragraph{Stage 2: reference resolution}
Stage~2 computes, for every contour $C$ with $\rho(C)\neq\bot$, the index of $D(C)$ in the list of traced contours. By (P2), $D(C)$ is the only contour that, during a visit of $q$, examines the background neighbour $w$. The neighbours examined at each visit are those of the cyclic range $[d_{in}+5,d_{out})$ (Section~\ref{sec:suzuki_abe}), where $d_{in}$ and $d_{out}$ can be recomputed from consecutive points of the contour; at the start pixel, the first visit examines the range $[\sigma_0,d_{out})$, with $\sigma_0=\mathrm{W}$ for an outer contour and $\mathrm{E}$ for a hole contour, and the closing visit is treated like any other. A single-pixel contour examines both sides. Stage~2 therefore proceeds as follows:
\begin{enumerate}
  \item A map $R$ over the pixels of the image stores, for every pixel $q$, the set of sides $\sigma\in\{\mathrm{W},\mathrm{E}\}$ such that $(q,\sigma)$ is the reference of some contour. Each pair $(q,\sigma)$ is the reference of at most one contour, because it is determined by the run of $q$ and the type of the contour that refers to it.
  \item The points of all contours are traversed. At each visit of a pixel $q$ by contour $D$, for every side $\sigma\in R(q)$ examined during that visit, the contour whose reference is $(q,\sigma)$ receives $D$ as its referenced contour.
\end{enumerate}
The traversal is divided into blocks containing the same number of points, regardless of contour boundaries, and the blocks are processed in parallel. Each entry $D(C)$ is written only by visits of a single contour, $D(C)$ itself, and always with the same value, so the result does not depend on the order in which the blocks are processed. The cost is linear in the total number of contour points, and a single long contour is split among several threads.

\paragraph{Stage 3: hierarchy assembly}
The contours are sorted in scan order. Each list $\mathcal{C}_k$ is already sorted, and two consecutive lists can only be out of order on the row $e_k$ that they share, so only the contours that start on those $K-1$ rows need to be sorted. The parents are then computed with~\eqref{eq:parent}, visiting the contours in scan order: since $D(C)$ precedes $C$, the value $\pi(D(C))$ is always available when $C$ is visited. This stage is sequential; apart from sorting the contours that start on the rows $e_k$, its cost is linear in the number of contours.

\begin{algorithm}[!tbp]
\renewcommand{\baselinestretch}{0.9}\selectfont
\SetAlgoLined
\DontPrintSemicolon
\KwIn{Contours $C_1,\dots,C_n$ from stage 1, with start pixels, types and references}
\KwOut{Contours in scan order and parent map $\pi$}
 \tcp{Stage 2 (parallel over blocks of points)}
 $R(q)\gets\emptyset$ for all $q$;\ \lFor{$i\gets1$ \KwTo $n$ \textbf{with} $\rho(C_i)=(q,\sigma)$}{$R(q)\gets R(q)\cup\{\sigma\}$}
 \ForEach{visit of a pixel $q$ by a contour $C_j$}{
   $S\gets$ sides of $q$ examined during the visit (from the entry and exit directions)\;
   \lForEach{$\sigma\in R(q)\cap S$}{$D(C_i)\gets C_j$, where $\rho(C_i)=(q,\sigma)$}
 }
 \tcp{Stage 3 (sequential)}
 sort the contours that start on the rows $e_0,\dots,e_{K-2}$ into scan order\;
 \ForEach{contour $C$ in scan order}{
   \uIf{$\rho(C)=\bot$}{$\pi(C)\gets\bot$}
   \uElseIf{$C$ and $D(C)$ have different types}{$\pi(C)\gets D(C)$}
   \lElse{$\pi(C)\gets\pi(D(C))$}
 }
 \caption{Stages 2 and 3 of TRUCO: hierarchy reconstruction.}
 \label{alg:hierarchy}
\end{algorithm}

\subsection{Synchronization and Memory Model}
\label{sec:memory}

Since all threads share $B$, the algorithm must also be correct with respect to the memory model of the language in which it is implemented. The memory models of the main languages with shared-memory threads, such as C, C++, Java and Rust, are built around the notion of a \emph{data race}: two conflicting accesses to the same memory location, at least one of them a write, that are not ordered by synchronization and are not both atomic. Programs without data races have well-defined behaviour and, if all their atomic accesses are sequentially consistent, behave as if their accesses were interleaved sequentially \cite{Adve1990,Boehm2008}. Programs with data races, in contrast, receive weak or no guarantees: in C and C++ their behaviour is undefined \cite{ISOCPP20,Batty2011}, and compilers may transform them in ways that break them even if they would be correct on the hardware \cite{Boehm2011benign}. TRUCO is free of data races by construction, independently of the language, and its requirements on the memory model are minimal, as shown next.

\paragraph{Who writes what}
In stage~1, each pixel of $B$ is written by at most one thread: the pixels of rows $[s_k,e_k)$ only by thread $k$ (Rule~1), and the frame and background pixels by no thread. The private rows and the per-stripe lists are accessed by a single thread. In stage~2, the map $R$ is only read, and each entry $D(C)$ is written only by the visits of the contour $D(C)$, possibly from two different blocks, always with the same value. The results of each stage are read by the next one only after its parallel loop has finished, which synchronizes with the completion of every task.

\paragraph{Concurrent accesses}
In stage~1, the pixels of rows $[s_k,e_k)$ can be accessed concurrently by thread $k$, which writes marks, and by other threads, which read them to test background/foreground while tracing, or to build the private copy of row $s_k=e_{k-1}$. It is sufficient that these accesses, the stores of thread $k$ and the loads of the other threads, are atomic. The remaining accesses to these pixels are reads by thread $k$ itself, when scanning its own rows and testing its own marks; they do not conflict with the loads of other threads (two reads never conflict) and are ordered with the thread's own stores by program order. In stage~2, the only possible concurrent writes are those of the same value to the same entry $D(C)$, which are also made atomic. Hence no pair of conflicting accesses is unordered and non-atomic, and the algorithm is data-race free.

\paragraph{Why no ordering is needed}
Atomicity guarantees that each load returns a value actually stored in that location, without tearing. TRUCO needs nothing more: no ordering between accesses to different locations, no fences and no read-modify-write operations. The reason is that no decision of a thread depends on a value written by another thread during the same stage: cross-thread loads in stage~1 only distinguish background from foreground, and every value that a pixel can hold has the same background/foreground status (Section~\ref{sec:decomposition}). The weakest atomic accesses offered by the language are enough, such as \emph{relaxed} atomics in C, C++ and Rust. In our implementation, the atomic accesses are applied directly to the pixels of the image, so no separate copy of it is needed. On x86-64, the architecture used in the experiments, relaxed atomic loads and stores of a byte are compiled to the same load and store instructions as non-atomic accesses, without fences or locked instructions; at the language level, their only effect is to forbid compiler transformations that are invalid under concurrency, such as inventing or merging stores.

\paragraph{Cache coherence}
Although no software synchronization is used, hardware cache coherence remains active: when a thread follows a contour into the rows of another stripe, it reads cache lines that the owner of those rows may be writing, which can cause coherence traffic, especially in images where a few contours span many stripes (Section~\ref{sec::stresstest}).

%%%%%%%%%%%%%%%%%%%%%%%%%%%%%%%%%%%%%%%%%%%%%%%%%%%%%%%%%%%%%%%%%%%%%%%%%%%%%
\section{Experimental Section}
\label{sec::experiments}
%%%%%%%%%%%%%%%%%%%%%%%%%%%%%%%%%%%%%%%%%%%%%%%%%%%%%%%%%%%%%%%%%%%%%%%%%%%%%

The experiments address five questions: (i) does the implementation produce exactly the output of Suzuki--Abe, contours and hierarchy, for every number of threads and in repeated executions; (ii) how does it scale with the number of threads, the resolution and the contour density; (iii) how does it behave on real high-resolution images; (iv) what is the cost of the hierarchy stages; and (v) where are its limits.

\subsection{Methods, Datasets and Protocol}
\label{sec:setup}

\paragraph{Compared methods}
The function \texttt{cv::findContours} of OpenCV, the most widely used implementation of Suzuki--Abe, does not follows exactly the original algorithm, for efficiency purposes. Its label image has one byte per pixel, so it stores only seven bits of each contour number, and the label found during the scan does not identify the contour $LNBD$ met before a start pixel. To find that contour, OpenCV keeps the contours that share each label value in a list and traces again, from their start pixels, the candidates whose bounding box contains the pixel, until it finds the one that passes through it. The cost of the hierarchy in OpenCV therefore depends on the number and the length of these candidates, in which every contour is traced exactly once. For this reason, the reference with the hierarchy is a sequential implementation of the Suzuki--Abe algorithm, written for this study, that follows the original paper \cite{SUZUKI198532}: a single raster scan over a 32-bit label image, in which every contour is traced once and its parent is obtained from $LNBD$ with the rule of Table~\ref{tab:parentrule}. It returns the same output as \texttt{cv::findContours} of OpenCV~4.13.0, the version used in all the experiments, in mode \texttt{RETR\_TREE} with \texttt{CHAIN\_APPROX\_NONE}, which we verified on 20\,000 random images and on samples of every dataset, and its source code is distributed with TRUCO. Its scan skips the runs of background pixels and the interiors of the runs of foreground pixels with vector instructions, as \texttt{cv::findContours} and TRUCO do, so that the comparison reflects the algorithms. We also report \texttt{cv::findContours} in mode \texttt{RETR\_TREE}, since it is the implementation used in practice. Without the hierarchy, the references are \texttt{cv::findContours} in mode \texttt{RETR\_LIST}, which runs the same algorithm on a one-byte label image with a vectorized scan, and the link-run method \cite{HE20091977} of OpenCV (\texttt{findContoursLinkRuns}, denoted RLE), which computes only a two-level hierarchy. All of them are single-threaded. TRUCO is reported without the hierarchy (stage~1) and with it (the three stages of Section~\ref{sec:ourmethod}; TRUCO-H in the figures), and each variant is compared with the methods that produce the same kind of output. TRUCO 1T denotes the execution with one thread and TRUCO best the execution with the best thread count for each image. Speedups are the time of the other method divided by the time of TRUCO.

The tiling method of Garcia-Molla and Alonso-Jord\'a \cite{Garcia_2023}, the most recent parallel contour tracer for CPUs, is not included. Its public implementation\footnote{\url{http://personales.upv.es/vmgarcia/borders_omp.tar.gz}} fixes the size of the image, the input file and the maximum number of borders per tile at compile time, and it writes an image in which the borders are marked, rather than the contour sequences and their hierarchy. A comparison on the output considered here would require rewriting a large part of it, and would therefore evaluate our implementation rather than that of the authors. GPU tracers address a different setting, in which the image is already in GPU memory (Section~\ref{sec:related_work}), and are not included.

\paragraph{Hardware and protocol}
We use an Intel Core i7-11700F (8 cores and 16 hardware threads, of which the benchmark uses 14) and an Intel Xeon Silver 4510 (two sockets of 12 cores; the benchmark uses the 24 hardware threads of one NUMA node). Pairs of values $a$/$b$ refer to these two CPUs, in this order. On both, the benchmark runs on cores reserved for it, at the base frequency and with real-time priority. All code is compiled with GCC~13 (\texttt{-O3}, no \texttt{-march}) and linked against OpenCV~4.13.0, and the number of threads of TRUCO, which is also the number of stripes $K$, is set with \texttt{cv::setNumThreads}; \ref{app:experiments} gives the details. For each image, the methods are executed in 4 rounds, in a random order in each round and 5 consecutive times each, and the reported time is the minimum of the 20 executions, which discards cold-start effects. The timed region contains one call to the method, including the allocation of the output, the conversion to a 32-bit label image for Suzuki--Abe and the creation of the padded one-byte copy $B$ for TRUCO; with the hierarchy, it includes the three stages and the construction of the output in the layout of OpenCV. Tables report the median time per image; speedups are computed per image and summarized by their median and inter-quartile range (IQR) over images \cite{Hoefler2015}.

\paragraph{Datasets}
We have employed a total of three datasets, one synthetic dataset, in which the contour density is controlled, and two datasets of real images, which represent two of the most common uses of contour extraction: the analysis of thresholded photographs and the vectorization of segmentation masks. Figure~\ref{fig:datasets} shows examples of each. The first one is a synthetic dataset of 40 binary images of $4000\times4000$ pixels with randomly placed small circles, generated with 8 object densities $p$ from $0.05$ to $0.75$ (5 images per density), which allows us to control the number of contours. The second one is the Liu4K dataset \cite{Liu4K}\footnote{\url{https://github.com/flyywh/Liu4K_benchmark}}, a collection of high-resolution photographs of varied scenes, which we use to reproduce the first stage of fiducial marker detection: the image is thresholded locally and all its contours are extracted to search for the markers among them \cite{ROMERORAMIREZ201838}. The same dataset has recently been used to evaluate the ArUco Nano marker detector \cite{GarridoJurado2026ArUcoNano}. We use the 924 images of its training set with a resolution of at least 16\,MP, resize them to that size, keeping their aspect ratio, and binarize them with the adaptive threshold of OpenCV (mean method, block size 15, constant 3, inverted output), which gives dense, irregular masks with many small components and holes. The images of these two datasets are also downscaled to 8, 4, 2 and 1\,MP before binarization, to study the effect of the resolution. The third one is the Inria Aerial Image Labeling dataset \cite{Maggiori2017Inria}\footnote{\url{https://project.inria.fr/aerialimagelabeling/}}, which represents the conversion of segmentation masks into polygons \cite{Xu2023HiSup,Zhang2025GCP}. We use the building masks of its training set, 180 binary images of $5000\times5000$ pixels (25\,MP), as they are. Unlike the previous ones, they are mostly background: an image has between 42 and 7003 contours (median 1483), 27\% of them holes, with up to four nesting levels, and its contour points are fewer than 1\% of its pixels.

\begin{figure*}[t!]
  \centering
  \begin{tikzpicture}
    \node[inner sep=0, anchor=south west] (img) at (0,0) {\includegraphics[width=\textwidth]{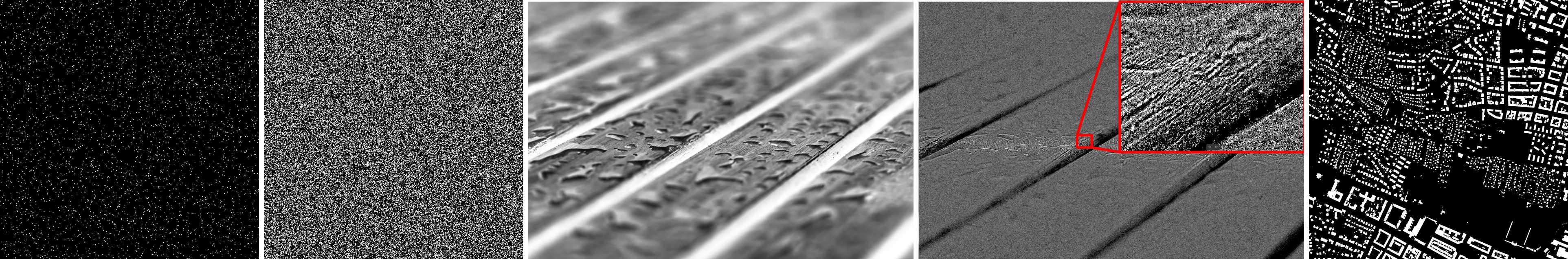}};
    % label centred under each panel (panel centres as fractions of the image width)
    \begin{scope}[x={(img.south east)}, y={(img.north west)}]
      \foreach \x/\l in {0.0826/a, 0.2509/b, 0.4596/c, 0.7087/d, 0.9174/e}
        \node[below, font=\footnotesize] at (\x,0) {(\l)};
    \end{scope}
  \end{tikzpicture}
  \caption{Examples of the datasets: (a) and (b) synthetic images with object density $p=0.05$ and $p=0.75$; (c) a Liu4K image and (d) its binary mask after the adaptive threshold, with a magnified detail; (e) an Inria building mask.}
  \label{fig:datasets}
\end{figure*}

\subsection{Correctness and Determinism}
\label{sec:validation}

A dedicated program compares the output of TRUCO with that of \texttt{cv::findContours} (\texttt{CHAIN\_APPROX\_NONE}) in the modes \texttt{RETR\_TREE} and \texttt{RETR\_LIST}. The exact comparison requires the same contours, in the same order, with the same point sequences and the same hierarchy vector; a topological comparison, independent of the order, and a check of the internal consistency of the hierarchy are also made (\ref{app:experiments}). Every image was processed with every thread count, several times, on the two CPUs and on a third machine with two Intel Xeon Silver 4316 (80 hardware threads on two sockets). Besides the images of every dataset and the stress images, 20\,000 random images of $1\times1$ to $120\times120$ pixels, with foreground densities between 0.05 and 0.95 and some of them dilated, eroded or replaced by their morphological gradient, exercise very short stripes, more stripes than rows and contours split between the tasks of stage~2. No difference was found in any of the 5.2 million executions (Table~\ref{tab:validation}). Since the exact comparison requires the same list, order and hierarchy, the output was also the same in every repetition, for every number of stripes and on the three machines. Under ThreadSanitizer \cite{Serebryany2009}, with TRUCO and OpenCV both instrumented, 52\,440 executions on 3000 random images and on 113 images of the other sets, with 1 to 20 threads, reported no data race in TRUCO and gave the same outputs.

\subsection{Speedup over Each Method}
\label{sec:summary}

Table~\ref{tab:summary} gives, for each dataset at its largest resolution, the time of every method and the speedup of TRUCO over each of them; the smallest speedup of each group is the speedup over the fastest competitor. With the hierarchy, TRUCO best is 9.8/17.4 times faster than the fastest of Suzuki--Abe and OpenCV \texttt{RETR\_TREE} on the synthetic images, 5.5/7.7 times on Liu4K and 4.7/10.0 times on the Inria masks. Without the hierarchy, it is 7.9/17.8, 7.1/13.0 and 2.3/7.5 times faster than the fastest of \texttt{RETR\_LIST} and RLE. On Liu4K, it processes a 16\,MP image with the hierarchy in 26.6/21.8\,ms, while OpenCV \texttt{RETR\_TREE} is more than 100 times slower. With one thread, TRUCO is already faster than every method with the same kind of output, except RLE on the Inria masks (Section~\ref{sec:ablation}), and the parallel speedup over TRUCO 1T is between 4.1 and 12.1 on the 16\,MP images. The two CPUs have similar sequential times, so the higher speedups of the Xeon come from its larger number of threads.

\begin{table*}[t!]
\centering
\caption{Median time per image (ms) of every method on each dataset at its largest resolution and, below it, the median speedup of TRUCO 1T / TRUCO best over that method. In the columns TRUCO best, the value in parentheses is the parallel speedup, TRUCO best over TRUCO 1T. OpenCV: \texttt{cv::findContours} with \texttt{RETR\_LIST} (without hierarchy) or \texttt{RETR\_TREE} (with hierarchy); $^\dagger$measured on every tenth image. RLE: \texttt{cv::findContoursLinkRuns}. S--A: Suzuki--Abe. A speedup above one means that TRUCO is faster. Times at every resolution are given in \ref{app:experiments}.}
\label{tab:summary}
\resizebox{\textwidth}{!}{%
\input{tables/summary.tex}}
\end{table*}

\subsection{Scaling with Threads, Resolution and Image Content}
\label{sec:scaling}

Figure~\ref{fig:scaling} shows the speedup as a function of the number of threads at every resolution. On both CPUs, it still grows at the largest thread count, so the measured range of threads does not reach the point at which additional threads stop reducing the time. With the hierarchy, it also grows with the resolution, from 5.5/7.4 on the synthetic images and 3.8/4.7 on Liu4K at 1\,MP, because larger images give each stripe more work with respect to the fixed costs of the parallel loop and of the hierarchy stages; without the hierarchy the dependence is weaker, except on Liu4K on the Xeon. Denser images give more parallel work: on the synthetic images at 16\,MP, from $p=0.05$ to $p=0.75$ the time of Suzuki--Abe grows from about 60 to more than 100\,ms, while that of TRUCO with all the threads stays between 2 and 15\,ms (\ref{app:experiments}). On Liu4K the speedup also depends on the image: at 16\,MP with the hierarchy its IQR is 4.8--6.3/6.4--9.1 and its range 2.9--18.8/2.9--42.9.

\begin{figure*}[t!]
  \centering
  \includegraphics[width=\textwidth]{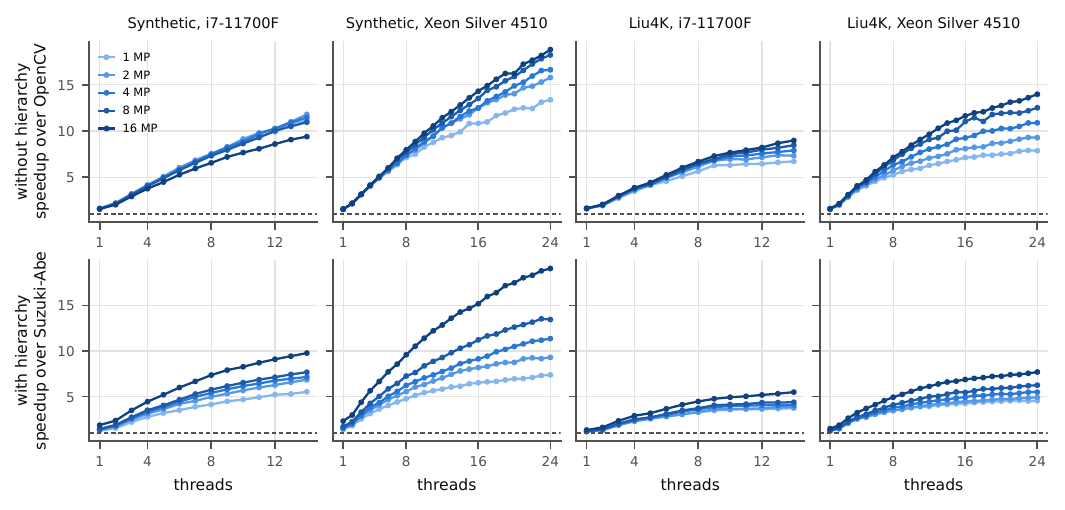}
  \caption{Median speedup of TRUCO as a function of the number of threads on the synthetic and Liu4K datasets, one curve per resolution: without hierarchy with respect to OpenCV \texttt{RETR\_LIST} (top) and with hierarchy with respect to Suzuki--Abe (bottom).}
  \label{fig:scaling}
\end{figure*}

\paragraph{Cost of the hierarchy}
Figure~\ref{fig:liu4k_phases} decomposes the time of TRUCO-H on Liu4K at 16\,MP. With one thread, stage~1 takes about three quarters of the time and stage~2 most of the rest. Both scale with the number of threads: with all of them, stage~2 takes 16--22\% of the time, and the sequential assembly becomes a larger fraction. Overall, TRUCO-H takes 1.7--2.2 times the time of TRUCO without the hierarchy, consistently with the second pass over the contour points of Section~\ref{sec:hierarchy}, but it remains well below Suzuki--Abe at every resolution.

\begin{figure*}[t!]
  \centering
  \includegraphics[width=\textwidth]{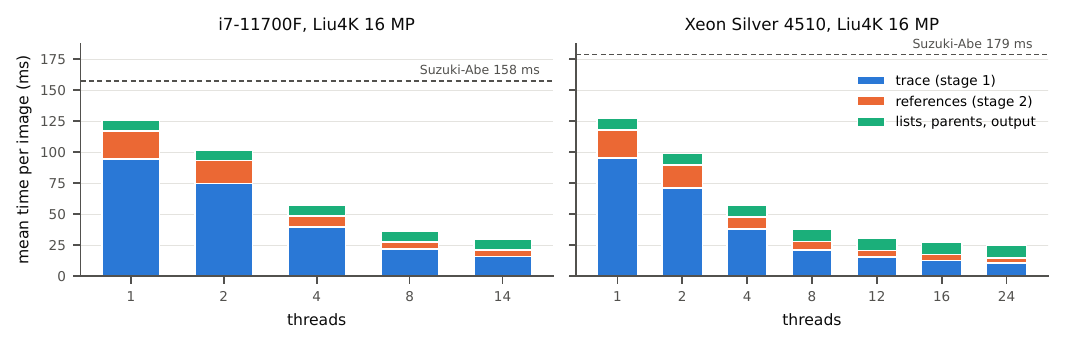}
  \caption{Liu4K dataset, 16\,MP: mean time per image of TRUCO with the hierarchy, decomposed into its stages, as a function of the number of threads. The dashed line is the mean time of Suzuki--Abe.}
  \label{fig:liu4k_phases}
\end{figure*}

\subsection{Sparse Segmentation Masks}
\label{sec:inria}

On the Inria masks, the speedup over Suzuki--Abe is much higher, 18.0/43.3, but a large part of it is sequential: with one thread, TRUCO-H is already 6.2/7.7 times faster (Table~\ref{tab:summary}). Since most pixels are background, the time of Suzuki--Abe is dominated by the initialization and the scan of its four-byte label image, while TRUCO and OpenCV use a one-byte image. OpenCV \texttt{RETR\_TREE}, which costs only 4\% more than \texttt{RETR\_LIST} here, is therefore the fastest sequential method with the hierarchy, and TRUCO best is 4.7/10.0 times faster than it. RLE, which has few runs to process on masks with few compact objects, is 1.5/1.3 times faster than TRUCO 1T, and TRUCO best is 2.3/7.5 times faster than RLE. The parallel speedup, 2.9/5.7 with the hierarchy, is lower than on Liu4K, and on the i7-11700F the time flattens beyond about 8 threads (Figure~\ref{fig:inria}). Stage~1 goes from 7.1/6.5 to 1.5/0.7\,ms, but the time outside the stages, mostly the copy of the input into the padded label image, only goes from 2.4/2.8 to 1.5/0.4\,ms, since the copy scales poorly on the i7-11700F; with all the threads it is 38/22\% of the time. The hierarchy costs 1.4/2.1 times the time without it, and only 1.2 times with one thread, because there are few contour points per pixel.

\begin{figure}[t!]
  \centering
  \includegraphics[width=\linewidth]{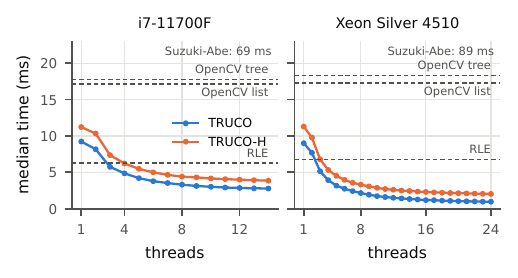}
  \caption{Inria building masks: median time per image of TRUCO without (TRUCO) and with (TRUCO-H) the hierarchy as a function of the number of threads. The dashed lines are the single-threaded OpenCV methods; Suzuki--Abe is out of the plotted range.}
  \label{fig:inria}
\end{figure}

\subsection{Single-Thread Performance}
\label{sec:ablation}

TRUCO separates the tracing of the contours from the computation of the hierarchy, so that the trace marks the pixels with four fixed values, which fit in a one-byte label image, instead of contour numbers. This section shows that this separation is advantageous by itself, independently of the parallel decomposition, and separates its contribution from that of the vectorized scan. To this end, Table~\ref{tab:ablation} separates the sources of the sequential gain with three variants of TRUCO 1T without the hierarchy, which differ only in the label image and in the scan and return the same contours. Even with a 32-bit label image and a scalar scan, TRUCO 1T is 11--38\%/11--49\% faster than Suzuki--Abe, which uses a 32-bit label image and a vectorized scan. This part of the gain therefore comes neither from the size of the label image nor from vectorization, but from the tracing procedure, in which TRUCO marks the pixels with four fixed values instead of contour numbers. The one-byte label image matters when the label image no longer fits in the caches: it reduces the time by a factor of 1.3--1.4 on Liu4K and of 1.6--1.8 on the synthetic images at 16\,MP, and by less than 5\% at 1\,MP. The vectorized scan reduces it by a further factor of 1.16--1.37. On the Inria masks both effects are larger (2.7/4.3 and 2.3/2.1). Since \texttt{cv::findContours} also uses a one-byte label image and a vectorized scan, TRUCO 1T is only 1.5--2.1 times faster than \texttt{RETR\_LIST} (Table~\ref{tab:summary}).

The columns with hierarchy of Table~\ref{tab:ablation} evaluate the second contribution of this paper, the reconstruction of the hierarchy from per-contour references, which differs from the computation of both Suzuki--Abe and OpenCV. Suzuki--Abe obtains the hierarchy during the scan, from the contour numbers stored in a 32-bit label image; TRUCO separates it from the trace, through the references of Section~\ref{sec:hierarchy}, and this is what allows its label image to keep one byte per pixel. The additional pass over the contour points increases the time of TRUCO 1T by a factor of only 1.2--1.3, less than what the one-byte label image saves, so that, with a single thread, TRUCO computes the full hierarchy 1.2--1.9/1.3--2.3 times faster than Suzuki--Abe on the synthetic and Liu4K images and 6.2/7.7 times faster on the Inria masks. It is also faster than OpenCV \texttt{RETR\_TREE}, which uses a one-byte label image but a slower computation of the hierarchy (Section~\ref{sec:setup}), on every set: by a factor of 2.4--3.9 on the synthetic images and on Liu4K at 1\,MP, of 27/30 on Liu4K at 16\,MP and of 1.7/1.8 on the Inria masks. The reconstruction from references is therefore not only what makes the parallel decomposition possible: with a single thread, it already makes the computation of the full hierarchy faster than with the two sequential methods.

\begin{table*}[t!]
    \centering
    \caption{Single-thread performance: median time per image (ms). S--A: Suzuki--Abe without the hierarchy; with it, it takes only 0--4\% more. TRUCO 1T without hierarchy: three variants that differ in the size of the label image and in the scan; the variant with an 8-bit label image and a vector scan is the actual TRUCO. With hierarchy: OpenCV \texttt{RETR\_TREE} ($^\dagger$measured on every tenth image) and TRUCO 1T with its three stages, from the main runs.}
    \label{tab:ablation}
    \input{tables/ablation.tex}\end{table*}

\subsection{Worst-Case Stress Tests}
\label{sec::stresstest}

TRUCO performs redundant work on contours that start in one stripe and cross many others (Section~\ref{sec:rules}). Two $4000\times4000$ images maximize it (Figure~\ref{fig:stress_tests}). The first consists of 2000 vertical lines one pixel wide, separated by one background pixel and spanning the whole height: all of them are owned by thread~0, and every other thread traces each line down and back up until it reaches its own first row, where it aborts the trace. The second is a single spiral, whose outer contour crosses all the stripes many times and can only be accepted by the thread that owns its top row.

TRUCO does not speed up on these images, but its time does not increase systematically with the number of threads either, since the redundant traces run concurrently with the owner, which traces the contours as the sequential algorithm does. On the stripes, the time rises from one to two threads (from 227 to 261\,ms and from 200 to 235\,ms), which we attribute to coherence traffic and to the shared memory bandwidth, and then stays flat; with all the threads, TRUCO-H is 1.09 times slower than Suzuki--Abe on the i7-11700F and 1.12 times faster on the Xeon. On the spiral, TRUCO-H with all the threads is 1.07 times faster than Suzuki--Abe on the i7-11700F; on the Xeon it is faster than Suzuki--Abe with every thread count, although its time varies irregularly (231--370\,ms) as the redundant traces interfere with the owner. Stage~2 scales normally on both images, since it splits the points of each contour among the threads (on the stripes, from 68\,ms with one thread to 10/5.6\,ms). In summary, on these two images, designed to maximize the redundant work, TRUCO obtains no parallel gain, but its time does not increase systematically with the number of threads, and TRUCO-H is at most 1.09 times slower than Suzuki--Abe.

\begin{figure}[t!]
    \centering
    \begin{subfigure}[b]{0.3\linewidth}
      \includegraphics[width=\linewidth]{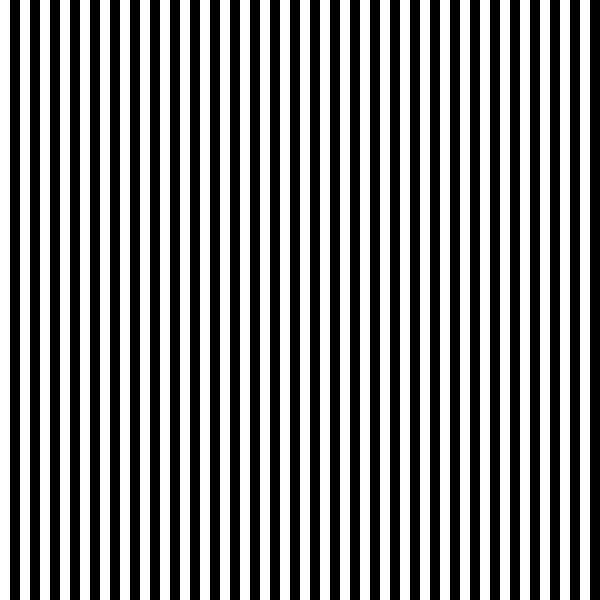}
      \caption{Stripes ($60\times60$ crop).}
    \end{subfigure}\hspace{0.05\linewidth}
    \begin{subfigure}[b]{0.3\linewidth}
      \includegraphics[width=\linewidth]{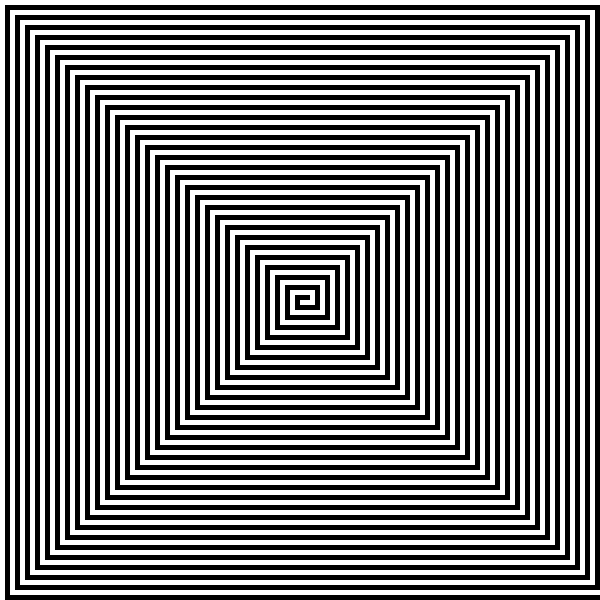}
      \caption{Spiral ($120\times120$ crop).}
    \end{subfigure}

    \vspace{0.3cm}
    \includegraphics[width=\linewidth]{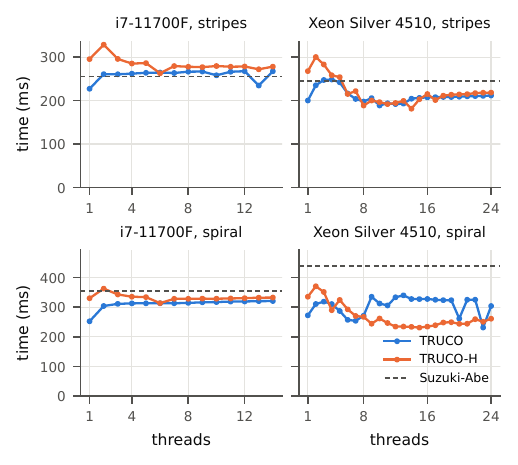}
    \caption{Stress images (top) and time of TRUCO without (TRUCO) and with (TRUCO-H) the hierarchy as a function of the number of threads (bottom). The dashed line is Suzuki--Abe with the hierarchy.}
    \label{fig:stress_tests}
\end{figure}

\subsection{Discussion}
\label{sec:discussion}

The experiments show that TRUCO returns the output of Suzuki--Abe in every execution and that, at the largest resolution of each dataset, it is faster than the sequential methods with the same kind of output, both with and without the hierarchy. Its speedup combines two gains of different origin. The first is sequential: with a single thread, TRUCO is already faster than Suzuki--Abe, and Section~\ref{sec:ablation} traces this gain to the tracing procedure, to the one-byte label image and to the vectorized scan. The second is parallel, 4.1--12.1 on the 16\,MP images (Table~\ref{tab:summary}). Both are part of the method: a sequential tracer can keep a one-byte label image, as OpenCV does, but a parallel one cannot store contour numbers in it, since the numbers of the contours found by the other stripes are not known when a stripe is scanned; the references of Section~\ref{sec:hierarchy} recover the hierarchy without them.

The parallel gain depends on the content of the image. It is smallest on the Inria masks, where there is little tracing work per pixel and the copy of the input into the padded label image, which scales poorly on the i7-11700F, is a large part of the time with all the threads. It also varies widely between the Liu4K images (Section~\ref{sec:scaling}), because the stripes are assigned statically and the work of each thread depends on how the contours are distributed over the rows; in the stress tests, where all the contours belong to one stripe, there is no parallel gain. A dynamic assignment of more stripes than threads would reduce this imbalance, and the output would remain deterministic for a fixed number of stripes $K$, since it does not depend on which thread processes each stripe.

%%%%%%%%%%%%%%%%%%%%%%%%%%%%%%%%%%%%%%%%%%%%%%%%%%%%%%%%%%%%%%%%%%%%%%%%%%%%%
\section{Conclusions}
\label{sec:conclusions}
%%%%%%%%%%%%%%%%%%%%%%%%%%%%%%%%%%%%%%%%%%%%%%%%%%%%%%%%%%%%%%%%%%%%%%%%%%%%%

This paper has presented TRUCO, a parallel contour-tracing algorithm for multi-core CPUs that returns exactly the output of the Suzuki--Abe algorithm: the same contours, in the same order, with the same hierarchy. It splits the image into horizontal stripes and lets every thread follow contours freely downwards, but restricts where each thread writes and which trace may accept each contour. Writing marks only on its own rows, plus a private copy of the next row, makes the behaviour of each thread independent of the others; accepting outer contours only from their top row and holes only from the row below it gives every contour a single owner, without communication. Because threads read the rows of other stripes only to tell background from foreground, which never changes, these reads can be relaxed atomic operations, and the algorithm is free of data races.

The hierarchy is recovered without contour numbers in the label image: each contour carries a reference, a pixel and a side known from the scan of its start row, which is resolved by a parallel pass over the contour points, and the parents follow from the rule of Suzuki--Abe. The method is modular: applications that only need the contours run the tracing stage alone, and those that need the hierarchy pay for an additional parallel pass over the contour points, which the sequential algorithm does not require but which does not offset the gain of parallel tracing. We have shown that the result is deterministic, that every contour is emitted exactly once and that the contours and the hierarchy equal those of Suzuki--Abe, and we have verified this equivalence on 5.2 million executions with 1 to 80 threads on three machines, and ThreadSanitizer reported no data race.

The experiments show that, at the largest resolution of each dataset, TRUCO is 2.3 to 17.8 times faster than the fastest sequential method with the same kind of output, both with and without the hierarchy. Part of this gain does not come from parallelism: separating the trace from the computation of the hierarchy already makes TRUCO with a single thread faster than the sequential methods in almost all cases. On images designed to maximize the redundant work of the stripe decomposition, TRUCO obtains no parallel gain, but its time does not increase systematically with the number of threads.

\section*{Declaration of generative AI and AI-assisted technologies in the manuscript preparation process}
During the preparation of this work, the authors used Claude (Anthropic) to assist in coding and writing. After using these tools, the authors reviewed and edited the content as needed and take full responsibility for the content of the published article.

\section*{Acknowledgements}
This research has been funded by the PID2023-147296NB-I00 project of the Ministry of Science, Innovation, and Universities of Spain.

\section*{CRediT authorship contribution statement}
\begin{itemize}
    \item \textbf{Rafael Muñoz-Salinas:} Conceptualization, Methodology, Software, Validation, Writing -- Original Draft, Project administration, Funding acquisition.
    \item \textbf{Francisco J. Romero-Ramirez:} Software, Investigation, Data Curation, Visualization, Writing -- Review \& Editing.
    \item \textbf{Manuel J. Marín-Jiménez:} Conceptualization, Formal analysis, Supervision, Writing -- Review \& Editing, Funding acquisition.
\end{itemize}

\section*{Data Availability Statement}
The source code of TRUCO is available at \url{https://github.com/rmsalinas/truco}. The synthetic and stress images and the raw timing files are available from the corresponding author upon reasonable request. Real-world images were obtained from the Liu4K dataset \cite{Liu4K} and from the Inria Aerial Image Labeling dataset \cite{Maggiori2017Inria}.

\appendix
\section{Correctness}
\label{app:proofs}

This appendix proves that TRUCO is deterministic and that its output is identical to that of Suzuki--Abe, with the notation of Sections~\ref{sec:suzuki_abe} and~\ref{sec:ourmethod}; in particular, $M(q)$ is the mark of pixel $q$ as seen by the thread under consideration (its value in $B$ for rows $<e_k$ and in the private row $P_k$ for row $e_k$). By (P2), the pixel sequence of a trace is determined by its start pixel, its start type and the background/foreground status of the pixels, never by marks. Correctness of stage~1 therefore depends only on the \emph{decisions} of the threads: which attempts are made and which are accepted.

\begin{proposition}[Determinism]\label{prop:det}
For a fixed input image and number of stripes $K$, the list $\mathcal{C}_k$ produced by thread $k$ in stage~1 is the same in every execution, whatever the scheduling and interleaving of the threads. Hence the whole output depends only on the image and on $K$.
\end{proposition}

\emph{Proof.} Thread $k$ reads three kinds of values. (i) The background/foreground status of arbitrary pixels of $B$. Background bytes are never written and foreground bytes only take values different from \BG{} (Section~\ref{sec:decomposition}); since these accesses are atomic, each load returns one of the values actually stored in that byte \cite{ISOCPP20}, so the result of every such test is fixed by the input. (ii) Marks on rows $[s_k,e_k)$ of $B$, which only thread $k$ writes (Rule~1), so each load returns the last value that thread $k$ stored. (iii) Marks in its private row $P_k$. All the branches of Algorithm~\ref{alg:parallel_worker} depend only on these values, so, by induction on its steps, thread $k$ executes the same sequence of operations in every execution. Stages~2 and~3 are functions of the output of stage~1: in stage~2 every entry $D(C)$ receives a single possible value (Section~\ref{sec:hierarchy}), and stage~3 is sequential. $\square$

\begin{theorem}[Contours]\label{thm:main}
For every $K\ge 1$ and every schedule, stage~1 returns each contour of the Suzuki--Abe algorithm (P4) exactly once, with the same pixel sequence and start pixel, and no other contour.
\end{theorem}

\emph{Proof.} By Proposition~\ref{prop:det}, each thread can be analysed as a sequential program that reads only its own marks. The proof has three steps: the start pixels of Suzuki--Abe are never blocked, no other start pixel is accepted, and each contour has exactly one owner. Fix thread $k$ with stripe $[s,e)$; every attempt it makes starts at a pixel that follows the start pixels of all its earlier attempts in raster order.

\emph{Step 1: the start pixels of (P4) are never blocked.} (a) If thread $k$ tests $p_A$, then $M(p_A)=\FG$: every earlier trace started at a pixel $q'\prec p_A$, which is not in $A$ because $p_A$ is the first pixel of $A$, and by (P2) it only visits pixels of $\comp(q')$. (b) If thread $k$ tests $p_Z$, then $M(p_Z)\neq\EAST$. An earlier trace may visit $p_Z$, but by (P3) it marks it as \EAST{} only if it follows a contour $C(\cdot,Z)$, that is, only if its first examined neighbour $w$ (the left neighbour of an outer attempt, the right neighbour of a hole attempt) is in $Z$. Since the attempt starts before $p_Z$, $w\preceq p_Z\prec h_Z$, which is impossible because $h_Z$ is the first pixel of $Z$. Moreover, the north-east neighbour of $p_Z$ is foreground (P1), so the hole attempt is made.

\emph{Step 2: only those start pixels are accepted.} Let an outer attempt at $q=(r,c)\in A$, whose left neighbour lies in $Z$, be accepted, so that its top row is $r$. If $Z$ were a hole of $A$, the trace would follow $C_Z$, whose top row is $\minrow(Z)-1<r$ (P1); hence $Z$ surrounds $A$, the trace follows $C_A$ and $\minrow(A)=r$. If $q\neq p_A$, the outer attempt at $p_A\prec q$, on the same row, was made earlier (Step~1), terminated without going above row $r$ (P2), and therefore was accepted and marked $q$, which would no longer be \FG{}: a contradiction.
Let a hole attempt at $q=(r,c)\in A$, whose right neighbour lies in $Z$, be accepted, so that its top row is $r-1\ge s$. If $Z$ surrounded $A$, then $\minrow(A)=r-1$, and the outer attempt at $p_A$, made earlier on row $r-1$, would have followed $C_A$ completely and marked $q$ as \EAST{} (P3), so the hole attempt at $q$ would not have been made. Hence $Z$ is a hole of $A$ and $\minrow(Z)=r$ (P1). If $q\neq p_Z$, the hole attempt at $p_Z\prec q$, on row $r$, was made earlier (Step~1), it followed $C_Z$ completely with top row $r-1$, and by (P3) it marked $q$ as \EAST{}: again a contradiction.
In both cases the accepted trace starts at the start pixel of (P4), with the same start type and first examined neighbour, so it has the same pixel sequence (P2).

\emph{Step 3: each contour is accepted by exactly one thread.} An accepted outer trace started on row $r$ has $\minrow=r\in[s,e)$, and an accepted hole trace started on row $r$ has $\minrow=r-1\in[s,e)$, so a contour can only be accepted by the thread whose stripe contains its top row, and by Step~2 that thread accepts it at most once. Conversely, the start pixel of a contour whose top row is in $[s,e)$ is on a row scanned by thread $k$ for its type (P1); the attempt is made (Step~1), terminates (P2) and never goes above the top row, so it is accepted. Since the stripes partition the rows, every contour is returned exactly once. $\square$

\begin{theorem}[Hierarchy]\label{thm:hier}
The output of stage~3 is the sequence of contours of Suzuki--Abe in scan order, together with the hierarchy $\pi$ of Section~\ref{sec:suzuki_abe}.
\end{theorem}

\emph{Proof.} By Theorem~\ref{thm:main} every contour is traced exactly once, so every referenced contour $D(C)$ is present and, by (P2), it is the only contour found for the reference $\rho(C)$ in stage~2. Equation~\eqref{eq:parent} gives the parent of every contour from its referenced contour, and the contours are visited in an order in which $\pi(D(C))$ is already known. Finally, the start pixels are those of Suzuki--Abe, so sorting by them gives the scan order. $\square$

Since, by Theorems~\ref{thm:main} and~\ref{thm:hier}, the output equals that of Suzuki--Abe, it does not depend on $K$ either.

\section{Experimental Details}
\label{app:experiments}

\paragraph{Hardware and system}
The i7-11700F has 62\,GB of RAM and runs Ubuntu 20.04.6 LTS; the benchmark runs on cores 1--7 (logical CPUs 1--7 and 9--15), and every other process is restricted to core 0 by setting its CPU affinity. The Xeon Silver 4510 has 251\,GB of RAM and runs Ubuntu 24.04.5 LTS; the benchmark runs on the cores and the memory of NUMA node~1, and every other process is restricted to node~0 through cgroup CPU sets. On both machines Turbo Boost is disabled (\texttt{intel\_pstate/no\_turbo}), the \texttt{performance} governor is selected and the frequency of every core is fixed at its base value, 2.5 and 2.4\,GHz. A script applies these settings and restores the original configuration at the end. The benchmark process runs under \texttt{SCHED\_FIFO} with priority 99; the single-threaded executions are pinned to one reserved core (logical CPU 2 and 12) with \texttt{pthread\_setaffinity\_np}, and in the multithreaded executions the operating system distributes the threads over the reserved cores.

\paragraph{Software}
All code is compiled in C++20 mode with \texttt{-O3 -DNDEBUG} and no \texttt{-march} option (GCC~13.1 on the i7-11700F and 13.3 on the Xeon) and linked against the official release of OpenCV~4.13.0, built in release mode with the \texttt{pthreads} parallel framework and the default CPU baseline (SSE3), with runtime dispatch of SSE4.1 to AVX-512 kernels. The run detection of TRUCO uses OpenCV universal intrinsics, which with this baseline are 128-bit vectors of 16 one-byte lanes. The atomic accesses of Section~\ref{sec:memory} use \texttt{std::atomic\_ref<unsigned char>} with \texttt{memory\_order\_relaxed} on the byte buffer of the image, and the code checks at compile time that they are lock-free.

\paragraph{Timing}
For the OpenCV methods, the input copy that they receive is made before the timer starts. With the hierarchy, the time of each stage of TRUCO is taken from the fastest of the 20 executions, so that the stages add up to the reported total. OpenCV \texttt{RETR\_TREE}, which is much slower on Liu4K, is measured once per round on every tenth Liu4K image (93 images). The resized Liu4K images are stored as JPEG with quality 95, the default of OpenCV.

\paragraph{Validation}
The topological comparison requires the same contours, each rotated to a canonical start (the smallest rotation starting at its raster-first point) and with its orientation, and the same parent for every contour; a contour returned twice is reported as an extra contour and a contour not returned as a missing one. Isolated foreground pixels are single-point outer contours, as in Suzuki--Abe. Table~\ref{tab:validation} gives the number of executions. The ThreadSanitizer runs used GCC~13.4 with \texttt{-fsanitize=thread} for TRUCO and OpenCV~4.13.0, on 3000 random images with 1 to 8 threads and on the two stress images, the 40 synthetic images at 1\,MP and 4 of them at 16\,MP, 47 Liu4K images at 1\,MP and 18 Inria masks with 1 to 20 threads, in the two modes. Two races inside the thread pool of OpenCV, between the pool and its workers, are suppressed; the suppressions match only when the top frame of the race is in the OpenCV pool, so a race in the code of TRUCO would still be reported.

\begin{table*}[!htbp]
\centering
\caption{Validation of TRUCO against \texttt{cv::findContours} (\texttt{CHAIN\_APPROX\_NONE}) in the modes \texttt{RETR\_TREE} and \texttt{RETR\_LIST}. Every image is processed in the two modes with every thread count of the machine (1--16 on the i7-11700F, 1--24 on the Xeon Silver 4510, 1--80 on the $2\times$ Xeon Silver 4316) and the given number of repetitions per thread count. Different: executions with an exact or a topological difference, on any machine.}
\label{tab:validation}
\begin{tabular}{lrr|rrr|r}
\toprule
 & & & \multicolumn{3}{c|}{\textbf{Executions}} & \\
\textbf{Images} & \textbf{\#Images} & \textbf{Rep.} & \textbf{i7-11700F} & \textbf{Xeon 4510} & \textbf{$2\times$Xeon 4316} & \textbf{Different} \\
\midrule
Random, up to $120\times120$      & 20\,000 & 1 & 640\,000 & 960\,000 & 3\,200\,000 & 0 \\
Stress, 16\,MP                    & 2       & 3 & 192      & 288         & 960      & 0 \\
Synthetic, 1--16\,MP (5 res.)     & 200     & 3 & 19\,200  & 28\,800     & 96\,000    & 0 \\
Liu4K, 1--16\,MP (5 res.)$^\ast$  & 465     & 1 & 14\,880  & 22\,320     & 74\,400    & 0 \\
Inria, 25\,MP                     & 180     & 3 & 17\,280  & 25\,920     & 86\,400    & 0 \\
\midrule
Total                             &         &   & 691\,552 & 1\,037\,328 & 3\,457\,760 & 0 \\
\bottomrule
\multicolumn{7}{l}{\footnotesize $^\ast$Every tenth image of the dataset (93 images) at each resolution.}
\end{tabular}
\end{table*}

\paragraph{Times at every resolution}
Tables~\ref{tab:synthetic} and~\ref{tab:liu4k} give the times of every method at every resolution. OpenCV \texttt{RETR\_TREE} takes 1.05--2.9 times the time of our Suzuki--Abe on the synthetic images, a factor that decreases with the resolution, and 2.6--19.5 times on Liu4K, where it grows with the resolution, so the overhead of its hierarchy depends strongly on the image content. Figure~\ref{fig:synthetic_density} shows the effect of the object density on the synthetic images at 16\,MP.

\begin{table*}[!htbp]
\centering
\caption{Synthetic dataset: median time per image (ms). For TRUCO best, the median speedup with respect to OpenCV (without hierarchy) or to Suzuki--Abe (S--A, with hierarchy) is given in parentheses. Other columns as in Table~\ref{tab:summary}.}
\label{tab:synthetic}
\resizebox{\textwidth}{!}{%
\input{tables/times_synthetic.tex}}
\end{table*}

\begin{table*}[!htbp]
\centering
\caption{Liu4K dataset (924 images): median time per image (ms). Columns as in Table~\ref{tab:synthetic}; $^\dagger$measured on every tenth image.}
\label{tab:liu4k}
\resizebox{\textwidth}{!}{%
\input{tables/times_liu4k.tex}}
\end{table*}

\begin{figure*}[!htbp]
  \centering
  \includegraphics[width=\textwidth]{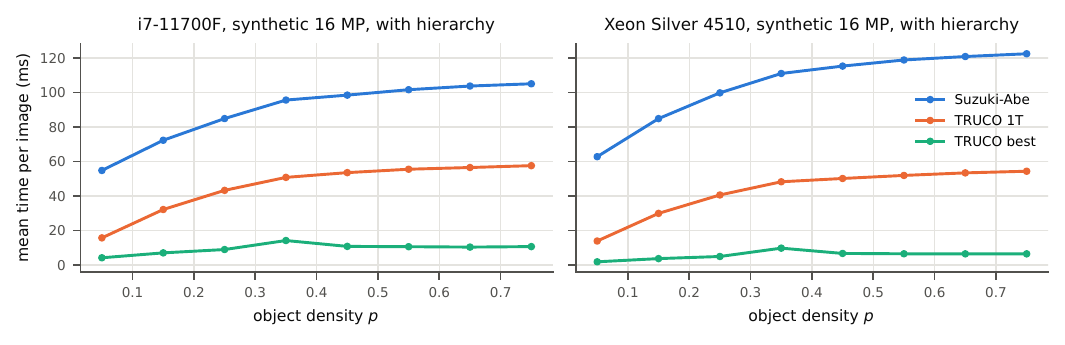}
  \caption{Synthetic dataset, 16\,MP, with hierarchy: mean time per image as a function of the object density $p$ (5 images per value).}
  \label{fig:synthetic_density}
\end{figure*}

\clearpage
\bibliographystyle{elsarticle-num}
\bibliography{bibliography}

\end{document}

%% file: tables/summary.tex
\begin{tabular}{ll|cccc|cccc}
\toprule
 & & \multicolumn{4}{c|}{\textbf{Without hierarchy}} & \multicolumn{4}{c}{\textbf{With hierarchy}} \\
\textbf{Dataset} & & OpenCV & RLE & TRUCO 1T & TRUCO best & OpenCV & S--A & TRUCO 1T & TRUCO best \\
\midrule
\multicolumn{10}{l}{\textit{i7-11700F}} \\
Synthetic, 16\,MP & time (ms) & 63.7 & 56.1 & 41.4 & 6.79 & 139.7 & 97.0 & 52.3 & 10.4 \\
 & speedup 1T / best & 1.6 / 9.4 & 1.3 / 7.9 &  & (5.9) & 2.6 / 12.6 & 1.9 / 9.8 &  & (4.9) \\
Liu4K, 16\,MP & time (ms) & 141.0 & 111.6 & 86.3 & 15.7 & 2398.3$^\dagger$ & 149.5 & 112.6 & 26.6 \\
 & speedup 1T / best & 1.6 / 9.1 & 1.3 / 7.1 &  & (5.5) & 26.6 / 115.8 & 1.3 / 5.5 &  & (4.1) \\
Inria, 25\,MP & time (ms) & 17.1 & 6.27 & 9.23 & 2.72 & 17.7 & 69.4 & 11.2 & 3.83 \\
 & speedup 1T / best & 1.5 / 6.3 & 0.7 / 2.3 &  & (3.4) & 1.6 / 4.7 & 6.2 / 18.0 &  & (2.9) \\
\midrule
\multicolumn{10}{l}{\textit{Xeon Silver 4510}} \\
Synthetic, 16\,MP & time (ms) & 58.3 & 62.1 & 37.7 & 3.47 & 120.4 & 113.3 & 49.2 & 6.38 \\
 & speedup 1T / best & 1.6 / 18.8 & 1.5 / 17.8 &  & (12.1) & 2.4 / 17.4 & 2.3 / 19.0 &  & (8.1) \\
Liu4K, 16\,MP & time (ms) & 140.0 & 133.7 & 87.4 & 10.6 & 2610.4$^\dagger$ & 169.5 & 114.6 & 21.8 \\
 & speedup 1T / best & 1.6 / 14.1 & 1.5 / 13.0 &  & (8.7) & 29.8 / 144.7 & 1.5 / 7.7 &  & (5.1) \\
Inria, 25\,MP & time (ms) & 17.3 & 6.75 & 9.00 & 0.95 & 18.3 & 88.6 & 11.3 & 2.03 \\
 & speedup 1T / best & 2.1 / 18.5 & 0.8 / 7.5 &  & (9.9) & 1.8 / 10.0 & 7.7 / 43.3 &  & (5.7) \\
\bottomrule
\end{tabular}

%% file: tables/ablation.tex
\begin{tabular}{l|c|ccc|cc}
\toprule
 & & \multicolumn{3}{c|}{\textbf{TRUCO 1T, without hierarchy}} & \multicolumn{2}{c}{\textbf{With hierarchy}} \\
\textbf{Images} & S--A & 32-bit, scalar & 8-bit, scalar & 8-bit, vector & OpenCV & TRUCO-H 1T \\
\midrule
\multicolumn{7}{l}{\textit{i7-11700F}} \\
Synthetic, 1\,MP & 11.1 & 8.03 & 7.80 & 6.66 & 35.0 & 9.20 \\
Synthetic, 16\,MP & 94.9 & 80.5 & 51.8 & 40.0 & 139.7 & 52.3 \\
Liu4K, 1\,MP & 11.4 & 8.76 & 8.52 & 7.35 & 40.6$^\dagger$ & 9.94 \\
Liu4K, 16\,MP & 144.3 & 129.6 & 101.3 & 85.7 & 2398.3$^\dagger$ & 112.6 \\
Inria, 25\,MP & 69.6 & 57.0 & 21.5 & 9.21 & 17.7 & 11.2 \\
\midrule
\multicolumn{7}{l}{\textit{Xeon Silver 4510}} \\
Synthetic, 1\,MP & 11.9 & 7.96 & 7.86 & 6.62 & 28.3 & 8.77 \\
Synthetic, 16\,MP & 114.5 & 91.2 & 52.0 & 37.9 & 120.4 & 49.2 \\
Liu4K, 1\,MP & 12.1 & 8.91 & 8.74 & 7.45 & 34.8$^\dagger$ & 9.90 \\
Liu4K, 16\,MP & 165.5 & 146.2 & 105.2 & 87.3 & 2610.4$^\dagger$ & 114.6 \\
Inria, 25\,MP & 89.6 & 80.8 & 19.0 & 8.91 & 18.3 & 11.3 \\
\bottomrule
\end{tabular}

%% file: tables/times_synthetic.tex
\begin{tabular}{l|ccccc|cccc}
\toprule
 & \multicolumn{5}{c|}{\textbf{Without hierarchy}} & \multicolumn{4}{c}{\textbf{With hierarchy}} \\
\textbf{Res.} & OpenCV & RLE & Garcia & TRUCO 1T & TRUCO best & S--A & OpenCV & TRUCO 1T & TRUCO best \\
\midrule
\multicolumn{10}{l}{\textit{i7-11700F}} \\
16\,MP & 63.7 & 56.1 & 132.6 & 41.4 & 6.79 ($\times$9.4) & 97.0 & 139.7 & 52.3 & 10.4 ($\times$9.8) \\
8\,MP & 40.7 & 29.3 & 77.5 & 26.4 & 3.82 ($\times$11.0) & 47.8 & 94.5 & 33.0 & 6.45 ($\times$7.7) \\
4\,MP & 23.5 & 16.4 & 51.1 & 15.1 & 2.17 ($\times$11.5) & 26.9 & 64.6 & 19.1 & 3.89 ($\times$7.2) \\
2\,MP & 15.9 & 10.4 & 38.9 & 9.92 & 1.35 ($\times$11.8) & 16.9 & 46.7 & 12.7 & 2.52 ($\times$6.8) \\
1\,MP & 11.5 & 7.29 & 34.3 & 6.91 & 1.03 ($\times$11.1) & 11.3 & 35.0 & 9.20 & 2.07 ($\times$5.5) \\
\midrule
\multicolumn{10}{l}{\textit{Xeon Silver 4510}} \\
16\,MP & 58.3 & 62.1 & 103.6 & 37.7 & 3.47 ($\times$18.8) & 113.3 & 120.4 & 49.2 & 6.38 ($\times$19.0) \\
8\,MP & 37.5 & 27.6 & 61.8 & 24.9 & 2.19 ($\times$18.3) & 52.3 & 79.8 & 30.9 & 4.01 ($\times$13.5) \\
4\,MP & 21.6 & 15.6 & 43.2 & 14.3 & 1.36 ($\times$16.8) & 27.7 & 52.1 & 18.1 & 2.56 ($\times$11.4) \\
2\,MP & 14.5 & 10.0 & 33.8 & 9.39 & 0.91 ($\times$15.8) & 17.8 & 37.1 & 12.1 & 1.98 ($\times$9.4) \\
1\,MP & 10.4 & 7.21 & 30.1 & 6.64 & 0.82 ($\times$13.4) & 12.3 & 28.3 & 8.77 & 1.67 ($\times$7.4) \\
\bottomrule
\end{tabular}

%% file: tables/times_liu4k.tex
\begin{tabular}{l|cccc|cccc}
\toprule
 & \multicolumn{4}{c|}{\textbf{Without hierarchy}} & \multicolumn{4}{c}{\textbf{With hierarchy}} \\
\textbf{Res.} & OpenCV & RLE & TRUCO 1T & TRUCO best & S--A & OpenCV & TRUCO 1T & TRUCO best \\
\midrule
\multicolumn{9}{l}{\textit{i7-11700F}} \\
16\,MP & 141.0 & 111.6 & 86.3 & 15.7 ($\times$9.1) & 149.5 & 2398.3$^\dagger$ & 112.6 & 26.6 ($\times$5.5) \\
8\,MP & 70.9 & 50.3 & 43.8 & 8.28 ($\times$8.6) & 67.2 & 679.1$^\dagger$ & 57.4 & 14.4 ($\times$4.5) \\
4\,MP & 39.3 & 27.8 & 23.9 & 4.88 ($\times$8.1) & 37.4 & 299.0$^\dagger$ & 32.1 & 8.72 ($\times$4.2) \\
2\,MP & 22.7 & 15.7 & 13.9 & 2.97 ($\times$7.7) & 21.4 & 114.4$^\dagger$ & 18.6 & 5.38 ($\times$3.9) \\
1\,MP & 12.0 & 7.85 & 7.50 & 1.73 ($\times$7.0) & 11.8 & 40.6$^\dagger$ & 9.94 & 3.10 ($\times$3.8) \\
\midrule
\multicolumn{9}{l}{\textit{Xeon Silver 4510}} \\
16\,MP & 140.0 & 133.7 & 87.4 & 10.6 ($\times$14.1) & 169.5 & 2610.4$^\dagger$ & 114.6 & 21.8 ($\times$7.7) \\
8\,MP & 69.5 & 54.7 & 43.7 & 5.38 ($\times$12.7) & 73.2 & 640.7$^\dagger$ & 57.2 & 11.1 ($\times$6.3) \\
4\,MP & 38.0 & 27.9 & 23.8 & 3.40 ($\times$11.2) & 38.9 & 270.5$^\dagger$ & 31.6 & 6.72 ($\times$5.6) \\
2\,MP & 21.8 & 15.7 & 13.8 & 2.20 ($\times$9.8) & 22.6 & 101.2$^\dagger$ & 18.5 & 4.39 ($\times$5.0) \\
1\,MP & 11.5 & 8.09 & 7.45 & 1.35 ($\times$8.4) & 12.6 & 34.8$^\dagger$ & 9.90 & 2.68 ($\times$4.7) \\
\bottomrule
\end{tabular}